\documentclass{article}
\usepackage[utf8]{inputenc}
\usepackage[left=2cm, right=3cm, top=2cm]{geometry}
\usepackage{graphicx}
\usepackage{subcaption}
\usepackage{amsmath}
\usepackage{amsfonts} 
\usepackage{algorithm}
\usepackage[noend]{algpseudocode}
\usepackage{bbm}
\newcommand{\citep}{\cite}

\usepackage{url}
\usepackage{amsmath,amsthm,amssymb,amsbsy}
\usepackage{mathdots}
\usepackage{paralist}
\usepackage{xcolor}
\usepackage{color}
\usepackage{graphicx}
\usepackage{algorithm,algpseudocode}
\usepackage{comment}

\usepackage{fancyhdr}
\usepackage{cite}
\usepackage{cleveref}
\usepackage{enumerate}

\newtheorem{thm}{Theorem}
\newtheorem{cor}{Corollary}
\newtheorem{prop}{Proposition}
\newtheorem{lem}{Lemma}
\theoremstyle{remark}
\newtheorem{remark}{Remark}

\newcommand{\R}{\mathbb{R}}

\newcommand{\vct}[1]{\boldsymbol{#1}}
\newcommand{\mtx}[1]{\boldsymbol{#1}}

\newcommand{\inner}[1]{\left<#1\right>}

\newcommand{\floor}[1]{\left\lfloor #1 \right\rfloor}
\newcommand{\ceil}[1]{\left\lceil #1 \right\rceil}
\newcommand{\eps}{\epsilon}

\newcommand{\calS}{\mathcal{S}}
\newcommand{\calT}{\mathcal{T}}

\newcommand{\calF}{\mathcal{F}}
\newcommand{\calC}{\mathcal{C}}

\newcommand{\va}{\vct{a}}
\newcommand{\vb}{\vct{b}}

\newcommand{\ve}{\vct{e}}

\newcommand{\vk}{\vct{k}}

\newcommand{\vm}{\vct{m}}
\newcommand{\vn}{\vct{n}}

\newcommand{\vu}{\vct{u}}
\newcommand{\vv}{\vct{v}}

\newcommand{\vx}{\vct{x}}
\newcommand{\vy}{\vct{y}}
\newcommand{\vz}{\vct{z}}

\newcommand{\mA}{\mtx{A}}

\newcommand{\mD}{\mtx{D}}

\newcommand{\mM}{\mtx{M}}

\newcommand{\mU}{\mtx{U}}
\newcommand{\mV}{\mtx{V}}

\newcommand{\mX}{\mtx{X}}
\newcommand{\mY}{\mtx{Y}}

\newcommand{\mPhi}{\mtx{\Phi}}
\newcommand{\mPsi}{\mtx{\Psi}}

\newcommand{\mId}{{\bf I}}

\graphicspath{{./figs/}}

\newlength{\imgwidth}
\newboolean{twoColVersion}
\setboolean{twoColVersion}{false}
\newcommand{\twoCol}[2]{\ifthenelse{\boolean{twoColVersion}} {#1} {#2} }

\newcommand{\calL}{\mathcal{L}}
\newcommand{\calN}{\mathcal{N}}
\newcommand{\calP}{\mathcal{P}}
\newcommand{\calY}{\mathcal{Y}}
\newcommand{\calE}{\mathcal{E}}

\DeclareMathOperator*{\supp}{\text{supp}}

\title{Neural Network Approximation of Continuous Functions in High Dimensions with Applications to Inverse Problems}
\author{Santhosh Karnik, Rongrong Wang, and Mark Iwen \thanks{Santhosh Karnik, Rongrong Wang, and Mark Iwen are with the Department of Computational Mathematics, Science, and Engineering at Michigan State University. Rongrong Wang and Mark Iwen are also with the Department of Mathematics at Michigan State University (e-mail: karniksa@msu.edu, wangron6@msu.edu, iwenmark@msu.edu).}
}

\begin{document}

\maketitle

\begin{abstract}
The remarkable successes of neural networks in a huge variety of inverse problems have fueled their adoption in disciplines ranging from medical imaging to seismic analysis over the past decade. However, the high dimensionality of such inverse problems has simultaneously left current theory, which predicts that networks should scale exponentially in the dimension of the problem, unable to explain why the seemingly small networks used in these settings work as well as they do in practice. To reduce this gap between theory and practice, we provide a general method for bounding the complexity required for a neural network to approximate a H\"older (or uniformly) continuous function defined on a high-dimensional set with a low-complexity structure. The approach is based on the observation that the existence of a Johnson-Lindenstrauss embedding $\mA\in\mathbb{R}^{d\times D}$ of a given high-dimensional set $\calS\subset\mathbb{R}^D$ into a low dimensional cube $[-M,M]^d$ implies that for any H\"older (or uniformly) continuous function $f:\calS\to\mathbb{R}^p$, there exists a H\"older (or uniformly) continuous function $g:[-M,M]^d\to\mathbb{R}^p$ such that $g(\mA\vx)=f(\vx)$ for all $\vx\in\calS$. Hence, if one has a neural network which approximates $g:[-M,M]^d\to\mathbb{R}^p$, then a layer can be added that implements the JL embedding $\mA$ to obtain a neural network that approximates $f:\calS\to\mathbb{R}^p$. By pairing JL embedding results along with results on approximation of H\"older (or uniformly) continuous functions by neural networks, one then obtains results which bound the complexity required for a neural network to approximate H\"older (or uniformly) continuous functions on high dimensional sets. The end result is a general theoretical framework which can then be used to better explain the observed empirical successes of smaller networks in a wider variety of inverse problems than current theory allows.
\end{abstract}

\section{Introduction}
At present various network architectures (NN, CNN, ResNet) achieve state-of-the-art performance for a broad range of inverse problems including matrix completion \citep{zheng2016neural,monti2017geometric,dziugaite2015neural,he2017neural}, image-deconvolution \citep{xu2014deep,kupyn2018deblurgan,nah2017deep}, low-dose CT-reconstitution \citep{chen2017low}, and electric and magnetic inverse Problems \citep{coccorese1994neural} (seismic analysis, electromagnetic scattering). However, since these problems are very high dimensional, classical universal approximation theory for such networks provides very pessimistic estimates of the network sizes required to learn such inverse maps (i.e., as being much larger than what standard computers can store, much less train). As a result, a gap still exists between the widely observed successes of networks in practice and the network size bounds provided by current theory in many inverse problem applications. 
 The purpose of this paper is to provide a refined bound on the size of networks in a wide range of such applications and to show that the network size is indeed affordable in many inverse problem settings. In particular, the bound developed herein depends on the model complexity of the domain of the forward map instead of the domain's extrinsic input dimension, and therefore is much smaller in a wide variety of model settings.

To be more specific, recall in most inverse problems one aims to recover some signal $\vx$ from its measurements $\vy = F(\vx)$.  Here $\vy$ and $\vx$ could both be high dimensional vectors, or even matrices and tensors, and $F$, which is called the forward map/operator, could either be linear or nonlinear with various regularity conditions depending on the application.  In all cases, however, recovering $\vx$ from $\vy$ amounts to inverting $F$. In other words, one aims to find the operator $F^{-1}$, that sends every measurement $\vy$ back to the original signal $\vx$.
Depending on the specific application of interest, there are various commonly considered forms of the forward map $F$.  For example, $F$ could be a linear map from high to low dimensions as in compressive sensing applications; $F$ could be a convolution operator that computes the shifted local blurring of an image in an image deblurring setting; $F$ could be a mask that filters out the unobserved entries of the data as in the matrix completion application; or $F$ could also be the source-to-solution map of a differential equation as in ODE/PDE based inverse problems. 

In most of these applications, the inverse operator $F^{-1}$  does not possess a closed-form expression. As a result, in order to approximate the inverse one commonly uses analytical approaches that involve solving, e.g., an optimization problem. Take sparse recovery as an example. With prior knowledge that the true signal $\vx\in \mathbb{R}^n$ is sparse, it is known that one can recover it from under-determined measurements, $\mathbb{R}^m \ni \vy=\mPhi\vx$ with $m<n$, by solving the optimization problem 
\[
\widehat{\vx} = \arg\min_{\vz} \|\vz\|_0, \quad \mPhi\vz = \vy.
\]
As a result, one can see that linear measurement map $F$ defined by $F(\vx) = \mPhi\vx = \vy$, with its domain restricted to the low-complexity domain of $s$-sparse vectors $\Sigma_s \subset \mathbb{R}^n$, has an inverse, $F^{-1}:  \mPhi(\Sigma_s) \rightarrow \Sigma_s$.  And, the minimizer $\widehat{\vx}$ above satisfies $\widehat{\vx} = F^{-1}(\vy) = F^{-1} (\mPhi \vx)= \vy$ for all $\vx \in \Sigma_s$. 

Note that traditional optimization-based approaches could be extremely slow for large-scale problems (e.g., for $n$ large above).  Alternatively, we can approximate the inverse operator by a neural network.  Amortizing the initial cost of an expensive training stage, the network can later achieve unprecedented speed over time at the test stage leading to better total efficiency over its lifetime.  To realize this goal, however, we need to first find a neural network architecture $f_{\theta}$, and train it to approximate $F^{-1}$, so that the  approximation error $\max_{\vy}\|f_{\theta}(\vy)- F^{-1}(\vy)\|=\|f_{\theta}(\vy)-\vx\|$ is small. The purpose of this paper is to provide a unified way to give a meaningful estimation of the size of the network that one can use in situations where the domain of $F$ is low-complexity, as is the case in, e.g., compressive sensing, low-rank matrix completion, deblurring with low-dimensional signal assumptions, etc..

\section{Main Results}
We begin by stating a few definitions. We say that a neural network \emph{$\eps$-approximates} a function $f$ if the function implemented by the neural network $\widehat{f}$ satisfies $\|\widehat{f}(\vx)-f(\vx)\|_{\infty} \le \eps$ for all $\vx$ in the domain of $f$. We say that a neural network architecture $\eps$-approximates a function class $\calF$ if for any function $f \in \calF$, there exists a choice of edge weights and node bias parameters such that the function $\widehat{f}$ implemented by the neural network with that choice of edge weights and node bias parameters satisfies $\|\widehat{f}(\vx)-f(\vx)\|_{\infty} \le \eps$ for all $\vx$ in the domain of $f$. For brevity, we refer to a feedforward neural network with at most $\calN$ nodes, $\calE$ edges\footnote{Throughout this paper, we use the term ``edges'' to denote the number of connections between nodes in a feedforward neural network. Many papers refer to these as ``weight parameters'' or just ``weights''.}, and $\calL$ layers as a $(\calN,\calE,\calL)$-FNN.  We also refer to a feedforward neural network architecture with at most $\calN$ nodes, $\calE$ edges, and $\calL$ layers as a $(\calN,\calE,\calL)$-FNN architecture. 
Similarly, we will refer to a convolutional neural network with at most $\calN$ nodes, $\calP$ parameters, and $\calL$ layers as a $(\calN,\calP,\calL)$-CNN.  And, we will also refer to a convolutional neural network architecture with at most $\calN$ nodes, $\calP$ parameters, and $\calL$ layers as a $(\calN,\calP,\calL)$-CNN architecture.

We say that a function $\Delta : [0,\infty) \to [0,\infty)$ is a modulus of continuity if it is non-decreasing and satisfies $\lim_{r \to 0^+}\Delta(r) = \Delta(0) = 0$. We say that a function $f$ between Euclidean spaces admits a modulus of continuity $\Delta$ if $\Delta$ is a modulus of continuity, and if $\|f(\vx)-f(\vx')\|_2 \le \Delta(\|\vx-\vx'\|_2)$ holds for all $\vx,\vx'$ in the domain of $f$. For brevity, we will refer to such a function as a $\calC(\Delta(r))$-function. 

For any constants $L > 0$ and $\alpha \in (0,1]$, we say that a function $f$ between Euclidean spaces is $(L,\alpha)$-H\"{o}lder if $f$ admits $\Delta(r) = Lr^{\alpha}$ as a modulus of continuity. We also say that a function $f$ is $\alpha$-H\"{o}lder if $f$ is $(L,\alpha)$-H\"{o}lder for some constant $L > 0$. 

Finally, for any positive integers $d < D$, any set $\calS \subset \R^D$, and any constant $\rho \in (0,1)$, we say that a matrix $\mA \in \R^{d \times D}$ is a $\rho$-JL (Johnson-Lindenstrauss) embedding of $\calS$ into $\mathbb{R}^d$ if  $$(1-\rho)\|\vx-\vx'\|_2 \le \|\mA\vx-\mA\vx'\|_2 \le (1+\rho)\|\vx-\vx'\|_2 \quad \text{holds for all} \quad \vx,\vx' \in \calS.$$ If we furthermore have $\mA(\calS) := \{\mA\vx : \vx \in \calS\} \subset \calT$, we say that $\mA$ is a $\rho$-JL embedding of $\calS$ into $\calT$. Intuitively, a $\rho$-JL embedding of $\calS$ into $\R^d$ maps $\calS$ from a high-dimensional space to a low-dimensional space without significantly distorting the Euclidean distances between points. \\ 

\noindent \textbf{Contributions and Related Work:}  Existing universal approximation theorems for various types of neural networks are mainly stated for functions defined on an $d$-dimensional cube. Our main contribution is to generalize these results to functions that admit a specified modulus of continuity (e.g., H\"{o}lder continuous functions) that are defined on arbitrary JL-embeddable subsets of very high dimensional Euclidean space.  We then demonstrate how our results can be applied to various example inverse problems in order to obtain reasonable estimates of the network sizes needed in each case we consider.


More explicitly, we show that if there exists a $\rho$-JL embedding of a high-dimensional set $\calS \subset \R^D$ into a low-dimensional cube $[-M,M]^d$, then we can use any neural network architecture which can $\eps$-approximate all $\left(\tfrac{L}{(1-\rho)^{\alpha}},\alpha \right)$-H\"{o}lder functions on $[-M,M]^d$ to construct a new neural network architecture which can $\eps$-approximate all $(L,\alpha)$-H\"{o}lder functions defined on $\calS$. To establish this, we show that if there exists a $\rho$-JL embedding $\mA \in \R^{d \times D}$ of $\calS \subset \R^D$ into $d$-dimensions, then for any $(L,\alpha)$-H\"{o}lder function $f : \calS \to \R^p$, there exists a $\left(\tfrac{L}{(1-\rho)^{\alpha}},\alpha \right)$-H\"{o}lder function $g : [-M,M]^d \to \R^p$ (where $M = \sup_{\vx \in \calS}\|\mA\vx\|_{\infty}$) such that $g(\mA\vx) = f(\vx)$ for all $\vx \in \calS$. Hence, if we have a neural network which can approximate $g : [-M,M]^d \to \R^p$, then we can compose it with a neural network which implements the JL embedding $\mA$ to obtain a neural network which approximates $f : \calS \to \R^p$. By pairing JL embedding existence results along with results on the approximation of H\"{o}lder functions by neural networks, we obtain results which bound the complexity required for a neural network to approximate H\"{o}lder functions on low-complexity high-dimensional sets. We can also generalize the above argument to extend results for functions with a specified modulus of continuity to higher dimensions in an analogous fashion.

The expressive power of neural networks is important in applications as a means of both guiding network architecture design choices, as well as for providing confidence that good network solutions exist in general situations.  As a result, numerous results about neural network approximation power have been established in recent years (see, e.g., \citep{zhou2020universality,petersen2020equivalence,yarotsky2022universal,Yarotsky18,lin2018resnet}). Most results concern the approximation of functions on all of $\R^D$, however, and yield network sizes that increase exponentially with the input dimension $D$.  As a result, the high dimensionality of many inverse problems leads to bounds from most of the existing literature which are too large to explain the observed empirical success of neural network approaches in such applications.

A similar high-dimensional scaling issue arises in many image classification tasks as well. Motivated by this setting  \citep{chen2019efficient} refined previous approximation results for ReLU networks, and showed that input data that is close to a low-dimensional manifold leads to network sizes that only grow exponentially with respect to the intrinsic dimension of the manifold. However, this improved bound relies on the data fitting a manifold assumption, which is quite strong in the setting of inverse problems.  For example, even the ``simple" compressive sensing/sparse recovery problem discussed above does not have a domain/range that forms a manifold (note that the intersections of $s$-dimensional subspaces in $\Sigma_s$ prevent it from being a manifold). Therefore, to study the expressive power of networks in the context of a class of inverse problems which is at least large enough to include compressive sensing, one needs to remove such strict manifold assumptions. Another mild issue with such manifold results is that the number of neurons also depends on the curvature of the manifold in question which can be difficult to estimate.  Furthermore, such curvature dependence is unavoidable for manifold results and needs to be incorporated into any valid bounds.\footnote{To see why, e.g., curvature dependence is unavoidable, consider any discrete training dataset in a compact ball. There always exists a 1-dimensional manifold, namely a curve, that goes through all the data points.  Thus, the mere existence of the 1-dimensional manifold does not mean the data complexity is low. Curvature information and other manifold properties matter as well!}

\begin{figure}
\centering
\includegraphics[height=3.5cm]{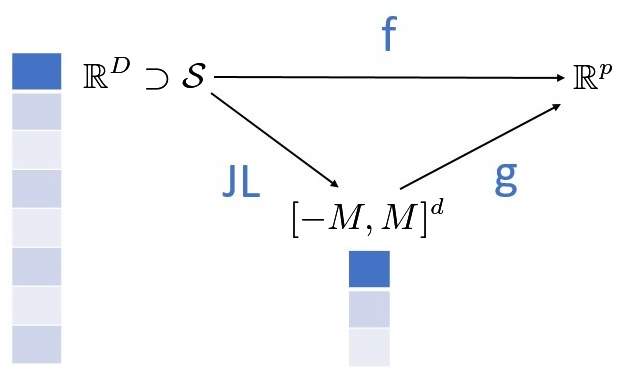}
\caption{If there exists a $\rho$-JL embedding of $\calS \subset \R^D$ into $[-M,M]^d$, then we can write the target function $f : \calS \to \R^p$ as $f = g \circ JL$ where $g : [-M,M]^d \to \R^p$. So, we can then construct a neural network approximation of $f$ by using a neural network approximation of $g$ and adding a layer to implement the JL embedding.}
\label{fig:1}
\end{figure}

The idea of using a JL embedding to reduce the dimensionality of the input before feeding it into a neural network was studied in \citep{wojcik2019training}. There, the authors perform experiments on the webspam\citep{webb2006introducing}, url \citep{ma2009identifying}, KDD2010-a, and KDD2010-b datasets \citep{yu2010feature}, which have input data with millions of dimensions. They show that a deep neural network with an untrained random projection layer that reduces the dimensionality down to $1000$ can improve on the state of the art results or achieve competitive performance on these datasets. Furthermore, the dimensionality reduction allows the neural network to be trained efficiently. However, \citep{wojcik2019training} did not provide any theoretical guarantees for the performance of a neural network with a linear dimensionality reduction layer.

We note that \citep{kratsios2022universal} have recently independently proposed a more general framework where a geometric deep learning model first applies a continuous and injective map $\phi$ to transform the input space to a manifold, after which a composition of exponential maps and DNNs is used to obtain the output. Applying a JL-transform to $\vx \in \mathcal{S}$ as done herein is similar, but there are some key differences. In our work, the JL-embedding is implemented by the first few layers of the neural network in order to compress the input data before it goes through the rest of the network. Their more general map $\phi$ is focused on feature extraction, and it is not implemented by a neural network. So, both $\phi$ and $\mathcal{S}$ must be known a priori in practice. More specifically, our results only require an upper bound on the intrinsic dimensionality of $\mathcal{S}$ in order to utilize theory guaranteeing the existence of linear JL-embeddngs with few rows.  Thus, neither detailed a priori knowledge of $\mathcal{S} \subset \mathbb{R}^D$ nor of the linear map is required herein.  As a result, our approach is significantly simpler to implement and apply in practice, and can often be utilized even when $S$ is only partially and/or approximately known.

We now state our main technical theorems which allow us to prove our main results. First, we give a theorem for extending feed-forward neural network approximation results from uniformly continuous functions on a low-dimensional hypercube to a high-dimensional set. See Appendix~\ref{sec:ProofMainContinuous} for the proof.

\begin{thm}
\label{thm:MainContinuousFNN}
Let $d < D$ be positive integers, and let $M > 0$, $\rho \in (0,1)$, and $\eps > 0$ be constants. Let $\Delta(r)$ be a modulus of continuity. Let $\calS \subset \R^D$ be a bounded subset for which there exists a $\rho$-JL embedding $\mA \in \R^{d \times D}$ of $\calS$ into $[-M,M]^d$. 

(a) Suppose that any $\calC(\sqrt{p}\Delta(\tfrac{r}{1-\rho}))$-function $g : [-M,M]^d \to \R^p$ can be $\eps$-approximated by a $(\calN,\calE,\calL)$-FNN. Then, any $\calC(\Delta(r))$ function $f : \calS \to \R^p$ can be $\eps$-approximated by a $(\calN+D,\calE+Dd,\calL+1)$-FNN.

(b) Furthermore, if there exists a single $(\calN,\calE,\calL)$-FNN architecture that can $\eps$-approximate every $\calC(\sqrt{p}\Delta(\tfrac{r}{1-\rho}))$-function $g : [-M,M]^d \to \R^p$, then there also exists another $(\calN+D,\calE+Dd,\calL+1)$-FNN architecture that can $\eps$-approximate every $\calC(\Delta(r))$-function $f : \calS \to \R^p$.
\end{thm}

We also provide a version of this theorem for extending approximation results for CNNs.  See Appendix~\ref{sec:ProofMainContinuous} for the proof.

\begin{thm}
\label{thm:MainContinuousCNN}
Let $d < D$ be positive integers, and let $M > 0$, $\rho \in (0,1)$, and $\eps > 0$ be constants. Let $\Delta(r)$ be a modulus of continuity. Let $\calS \subset \R^D$ be a bounded subset for which there exists a $\rho$-JL embedding $\mA \in \R^{d \times D}$ of $\calS$ into $[-M,M]^d$ which is of the form $\mA = \mM\mD$, where $\mM$ is a partial circulant matrix, and $\mD$ is a diagonal matrix with $\pm 1$ on its diagonal. 

(a) Suppose that any $\calC(\sqrt{p}\Delta(\tfrac{r}{1-\rho}))$-function $g : [-M,M]^d \to \R^p$ can be $\eps$-approximated by a $(\calN,\calP,\calL)$-CNN. Then, any $\calC(\Delta(r))$-function $f : \calS \to \R^p$ can be $\eps$-approximated by a $(\calN+4D,\calP+2D,\calL+4)$-CNN.

(b) Furthermore, if there exists a single $(\calN,\calP,\calL)$-CNN that can $\eps$-approximate every $\calC(\sqrt{p}\Delta(\tfrac{r}{1-\rho}))$-function $g : [-M,M]^d \to \R^p$, then there also exists another $(\calN+4D,\calP+2D,\calL+4)$-CNN that can $\eps$-approximate every $\calC(\Delta(r))$-function $f : \calS \to \R^p$.
\end{thm}

We also provide improvements for both of the above theorems which state that for the special case of H\"{o}lder functions, whose modulus of continuity is $\Delta(r) = Lr^{\alpha}$, the $\sqrt{p}$ factor in $\sqrt{p}\Delta(\tfrac{r}{1-\rho})$ can be removed. See Appendix~\ref{sec:ProofMainHolder} for the proofs.

\begin{thm}
\label{thm:MainHolderFNN}
Let $d < D$ be positive integers, and let $L > 0$, $M > 0$, $\alpha \in (0,1]$, $\rho \in (0,1)$, and $\eps > 0$ be constants. Let $\calS \subset \R^D$ be a bounded subset for which there exists a $\rho$-JL embedding $\mA \in \R^{d \times D}$ of $\calS$ into $[-M,M]^d$. 

(a) Suppose that any $\left(\tfrac{L}{(1-\rho)^{\alpha}},\alpha \right)$-H\"{o}lder function $g : [-M,M]^d \to \R^p$ can be $\eps$-approximated by a $(\calN,\calE,\calL)$-FNN. Then, any $(L,\alpha)$-H\"{o}lder function $f : \calS \to \R^p$ can be $\eps$-approximated by a $(\calN+D,\calE+Dd,\calL+1)$-FNN. 

(b) Furthermore, if there exists a single $(\calN,\calE,\calL)$-FNN architecture that can $\eps$-approximate every $\left(\tfrac{L}{(1-\rho)^{\alpha}},\alpha \right)$-H\"{o}lder function $g : [-M,M]^d \to \R^p$, then there also exists another $(\calN+D,\calE+Dd,\calL+1)$-FNN architecture that can $\eps$-approximate every $(L,\alpha)$-H\"{o}lder function $f : \calS \to \R^p$.
\end{thm}

\begin{thm}
\label{thm:MainHolderCNN}
Let $d < D$ be positive integers, and let $L > 0$, $M > 0$, $\alpha \in (0,1]$, $\rho \in (0,1)$, and $\eps > 0$ be constants. Let $\calS \subset \R^D$ be a bounded subset for which there exists a $\rho$-JL embedding $\mA \in \R^{d \times D}$ of $\calS$ into $[-M,M]^d$ which is of the form $\mA = \mM\mD$, where $\mM$ is a partial circulant matrix, and $\mD$ is a diagonal matrix with $\pm 1$ on its diagonal.

(a) Suppose that any $\left(\tfrac{L}{(1-\rho)^{\alpha}},\alpha\right)$-H\"{o}lder function $g : [-M,M]^d \to \R^p$ can be $\eps$-approximated by a $(\calN,\calP,\calL)$-CNN. Then, any $(L,\alpha)$-H\"{o}lder function $f : \calS \to \R^p$ can be $\eps$-approximated by a $(\calN+4D,\calP+2D,\calL+4)$-CNN. 

(b) Furthermore, if there exists a single $(\calN,\calP,\calL)$-CNN architecture that can $\eps$-approximate every $\left(\tfrac{L}{(1-\rho)^{\alpha}},\alpha \right)$-H\"{o}lder function $g : [-M,M]^d \to \R^p$, then there also exists another $(\calN+4D,\calP+2D,\calL+4)$-CNN architecture that can $\eps$-approximate every $(L,\alpha)$-H\"{o}lder function $f : \calS \to \R^p$.
\end{thm}

\begin{remark} These theorems ensure that the network size for approximating $f$ grows exponentially with the compressed dimension $d$ instead of growing exponentially with the input dimension $D$. The task now reduces to making the compressed dimension $d$ as small as possible while still ensuring that a $\rho$-JL embedding of $\calS$ into $[-M,M]^d$ exists.
\end{remark}
\begin{remark}These theorems are quite general as parts (a) and (b) are not restricted to any particular type of network or activation function. In Section~\ref{sec:MainResults}, we provide three corollaries of these theorems which establish the expressive power of the feedforward and convolutional neural networks. 
\end{remark}
\begin{remark}
If an inverse operator is H\"{o}lder continuous and there exists a $\rho$-JL embedding of the set of possible observations $\calS$ into $d$ dimensions, then the theorem gives us a bound on the complexity of a neural network architecture required to approximate the inverse operator.
\end{remark}
\begin{remark}
A key ingredient in the above theorems is that for any uniformly (or H\"{o}lder) continuous function $f : \calS \to \R^p$ and any $\rho$-JL embedding $\mA \in \R^{d \times D}$ of $\calS$ into a hypercube $[-M,M]^d$, there exists a uniformly (or H\"{o}lder) continuous function $g : [-M,M]^d \to \R^p$ such that $g(\mA\vx) = f(\vx)$ for all $\vx \in \calS$. Unfortunately, this is not the case if we replace uniform (or H\"{o}lder) continuity with differentiability. In Appendix~\ref{sec:Example}, we provide an example of a set $\calS \subset \R^D$ and a smooth function $f : \calS \to \R$ for which there is no $\rho$-JL embedding $\mA \in \R^{d \times D}$ (with $d < D$) of $\calS$ into a hypercube $[-M,M]^d$ and differentiable function $g : [-M,M]^d \to \R^p$ such that $g(\mA\vx) = f(\vx)$ for all $\vx \in \calS$. As such, we are not able to use our main idea to extend approximation results about differentiable functions.
\end{remark}
\subsection{JL embeddings, covering numbers and Gaussian width}

As the existence of the JL map is a critical assumption of our theorem, in this section, we discuss the sufficient conditions for this assumption to hold.  In addition, we also care about the structures of the JL maps, as they will end up being the first layer of the final neural network. For example, if the neural network is of convolution type, we need to make sure that a circulant JL matrix exists.

\textbf{Existence of $\rho$-JL maps}: It is well-known that for finite sets $\mathcal{S}$, the existence of a $\rho$-JL embedding can be guaranteed by the Johnson-Lindenstrauss Lemma. For sets $\mathcal{S}$ with infinite cardinally, the Johnson-Lindenstrauss lemma cannot be directly used. In the following proposition, we extend the Johnson-Lindenstrauss lemma from a finite set of $n$ points to a general set $\calS$.  See Appendix~\ref{sec:ProofJLFiniteToInfinite} for its proof.

\begin{prop}
\label{prop:JLFiniteToInfinite}
Let $\rho\in (0,1)$. For $\mathcal{S} \subseteq \mathbb{R}^D$, define $$U_{\calS} := \overline{\left\{\dfrac{\vx-\vx'}{\|\vx-\vx'\|_2} \ : \ \vx, \vx' \in \calS ~s.t.~ \vx \neq \vx'~\right\}}$$ to be the closure of the set of unit secants of $\calS$, and $\mathcal{N}(U_{\mathcal{S}}, \|\cdot\|_2, \delta)$ to be the covering number of $U_{\mathcal{S}}$ with $\delta$-balls. 
Then, there exists a set $\calS_1$ with $|\calS_1| = 2\mathcal{N}(U_{\mathcal{S}}, \|\cdot\|_2, \delta)$ points such that if a matrix $\mA \in \R^{d \times D}$ is a $\rho$-JL embedding of $\calS_1$, then $\mA$ is also a $(\rho+2\|\mA\|_2\delta)$-JL embedding of $\calS$.
\end{prop}
The proposition guarantees that whenever we have a JL-map for finite sets, we can extend it to a JL-map for infinite sets with similar level of complexity measured in terms of the covering numbers. There are many known JL-maps for finite sets that we can extend from, including sub-Gaussian matrix \citep{matouvsek2008variants}, Gaussian circulant matrices with random sign flip \citep{cheng2014new}, etc. 
We present some of the related results here.

\begin{prop}[\citep{matouvsek2008variants}]
\label{prop:SubgaussianJL}
Let $\vx_1,\ldots,\vx_n \in \R^D$. Let $\rho \in (0,\tfrac{1}{2})$ and $\beta \in (0,1)$. Let $\mA \in \R^{d \times D}$ be a random matrix whose entries are i.i.d. from a subgaussian distribution with mean $0$ and variance $1$. Then, there exists a constant $C > 0$ depending only on the subgaussian distribution such that if $d \ge C\rho^{-2}\log\tfrac{n}{\beta}$, then $\tfrac{1}{\sqrt{d}}\mA$ will be a $\rho$-JL embedding of $\{\vx_1,\ldots,\vx_n\}$ with probability at least $1-\beta$.
\end{prop}

\begin{prop}[Corollary 1.3 in \citep{cheng2014new}]
\label{prop:CirculantJL}
Let $\vx_1,\ldots,\vx_n \in \R^D$. Let $\rho \in (0,\tfrac{1}{2})$, and let $d = O(\rho^{-2}\log^{1+\alpha}n)$ for some $\alpha > 0$. Let $\mA = \tfrac{1}{\sqrt{d}}\mM\mD$ where $\mM \in \R^{d \times D}$ is a random Gaussian circulant matrix and $\mD \in \R^{D \times D}$ is a random Rademacher diagonal matrix. Then, with probability at least $\tfrac{2}{3}\left(1-(D+d)e^{-\log^{\alpha}n}\right)$, $\mA$ is a $\rho$-JL embedding of $\{\vx_1,\ldots,\vx_n\}$.
\end{prop}
Note that the $\alpha$ in the proposition can be set to be any positive number making the probability of failure less than 1.
Combining the results of Propositions~\ref{prop:SubgaussianJL} and \ref{prop:CirculantJL} with Proposition~\ref{prop:JLFiniteToInfinite}, we have the following existence result for the JL map of an arbitrary set $\mathcal{S}$.  See Appendix~\ref{sec:ProofJLCoveringNumber} for its proof.

\begin{prop} 
\label{prop:JLCoveringNumber}
Let $\rho \in (0,\tfrac{1}{2})$ be a constant. For $\mathcal{S} \subseteq \mathbb{R}^D$, let $\mathcal{N}(U_{\mathcal{S}}, \|\cdot\|_2, \delta)$ to be the covering number with $\delta$-balls of the unit secant $U_{\mathcal{S}}$ of $\mathcal{S}$ defined in Proposition~\ref{prop:JLFiniteToInfinite}. Then 

(a) If $D \geq d \gtrsim \rho^{-2}\log \mathcal{N}(U_{\mathcal{S}}, \|\cdot\|_2, \frac{\rho}{4\sqrt{3D}})$, then there exists a matrix $\mA \in \R^{d \times D}$ which is a $\rho$-JL embedding of $\mathcal{S}$.

(b) If $D \geq d \gtrsim \rho^{-2}\log(4D+4d)\log \mathcal{N}(U_{\mathcal{S}}, \|\cdot\|_2, \frac{\rho}{4\sqrt{3D}})$, then there exists a matrix $\mA \in \R^{d \times D}$ in the form of $\mM\mD$ and of size $d\times D$ that works as $\rho$-JL map for $\mathcal{S}$, where $\mM$ is a partial circulant matrix and $\mD$ is a diagonal matrix with $\pm 1$ on its diagonal.
\end{prop}

The above proposition characterizes the compressibility of a set $\mathcal{S}$ by a JL-mapping in terms of the covering number. Alternatively, one can also characterize it using the Gaussian width. For example, in \citep{Iwen22} it is shown using methods from \citep{vershynin2018high} that if the set of unit secants of $\calS$ has a low Gaussian width, then with high probability a subgaussian random matrix with provide a low-distortion linear embedding, and the dimension $d$ required scales quadratically with the Gaussian width of the set of unit secants of $\calS$. 

\begin{prop}[Corollary 2.1 in \citep{Iwen22}]
\label{prop:JLGaussianWidth}
Let $\rho,\beta \in (0,1)$ be constants. Let $\mA \in \R^{d \times D}$ be a matrix whose rows $\va_1^T,\ldots,\va_d^T$ are independent, isotropic ($\mathbb{E}[\va_i\va_i^T] = \mId$), and subgaussian random vectors. Let $\calS \subset \R^D$, and let 
$$\omega(U_{\calS}) := \mathbb{E}\sup_{\vu \in U_{\calS}}\inner{\vu,\vz}, \quad \vz \sim \text{Normal}(0,\mId)$$ be the Gaussian width of $U_{\calS}$. Then, there exists a constant $C > 0$ depending only on the distribution of the rows of $\mA$ such that if $$d \ge \dfrac{C}{\rho^2}\left(\omega(U_{\calS}) + \sqrt{\log\tfrac{2}{\beta}}\right)^2,$$ then $\tfrac{1}{\sqrt{d}}\mA$ is a $\rho$-JL embedding of $\calS$ with probability at least $1-\beta$.
\end{prop}

If $\mathcal{S}$ is known, one can use either the log-covering number (Proposition~\ref{prop:SubgaussianJL}) or the Gaussian width (Proposition~\ref{prop:CirculantJL}) to compute a lower bound for $d$. If one only has samples from $\mathcal{S}$, one may still estimate the covering number. In \citep{kegl2002intrinsic}, the authors demonstrate a practical method for estimating the intrinsic dimension of a set by using a greedy algorithm to estimate the log-packing number for several different radii $\delta$, and then extrapolating the linear region of the graph of the log-packing number vs. $\log \delta$ to estimate the packing number at finer radii. Then using the fact that the covering numbers of a set may be bounded by its packing numbers, one may in turn obtain bounds for the log-covering number.

\subsection{Universal approximator neural networks for uniformly/H\"{o}lder continuous functions on $d$-dimensional cubes}
In Theorems 1 and 2, we showed that  with the help of JL, approximation rate of neural networks for functions defined on an arbitrary set $\mathcal{S}$ can be derived from their approximation rates for functions defined on the cube $[-M,M]^d$. In this section, we review known results for the latter, so that they can be used in combination with Theorems 1 and 2 to provide useful approximation results for network applications to various inverse problems. Specifically, we review two types of universal approximation results for functions defined on the cube $[-M,M]^d$. One is for Feedforward ReLU networks and the other is for Resnet type CNNs.\\

\noindent \textbf{Feedforward ReLU network:} The fully connected feedforward neural network with ReLU activation is known to be a universal approximator of any uniformly continuous function on the box $[-M,M]^d$. Moreover, for such networks, the non-asymptotic approximation error has also been established, allowing us to get an estimate of the network size. The proposition below is a variant of Proposition 1 in \citep{Yarotsky18}, and the proof uses an approximating function that uses the same ideas as in \citep{Yarotsky18}.  See Appendix~\ref{sec:ProofContinuousFNNlowdim} for its proof.

\begin{prop}
\label{prop:ContinuousFNNlowdim}

Let $d$ and $p$ be positive integers, and let $M > 0$ be a constant. Let $\Delta(r)$ be a modulus of continuity. Then, for any positive integer $N$, there exists a ReLU NN architecture with at most $$(p+C_1)(2N+1)^d \ \text{edges}, \ C_2(2N+1)^d + p \ \text{nodes}, \ \text{and} \ \ceil{\log_2(d+1)} + 2 \ \text{layers}$$ that can $\Delta(\tfrac{M\sqrt{d}}{N})$-approximate the class of $\calC(\Delta(r))$-functions $g : [-M,M]^d \to \R^p$. Here, $C_1,C_2 > 0$ are universal constants. Also for each edge of the ReLU NN, the corresponding weight is either independent of $g$, or is of the form $g_{i}(\vy)$ for some fixed $\vy \in [-M,M]^d$ and coordinate $i = 1,\ldots,p$.
\end{prop}

By setting $\Delta(r) = Lr^{\alpha}$ for some constants $L > 0$ and $\alpha \in (0,1]$, and setting $N = \ceil{\tfrac{M\sqrt{d}}{(\eps/L)^{1/\alpha}}}$ for some $\epsilon > 0$, we get the following corollary for approximating H\"{o}lder functions. 

\begin{cor}
\label{cor:HolderFNNlowdim}
Let $d$ and $p$ be positive integers, and let $L, M, \eps > 0$ and $\alpha \in (0,1]$ be constants. Then, there exists a ReLU NN architecture with at most $$(p+C_1)\left(2\ceil{\dfrac{M\sqrt{d}}{(\eps/L)^{1/\alpha}}}+1\right)^d \ \text{edges}, \ C_2\left(2\ceil{\dfrac{M\sqrt{d}}{(\eps/L)^{1/\alpha}}}+1\right)^d + p \ \text{nodes}, \ \text{and} \ \ceil{\log_2(d+1)} + 2 \ \text{layers}$$ that can $\eps$-approximate the class of $(L,\alpha)$-H\"{o}lder functions $g : [-M,M]^d \to \R^p$. Again, $C_1,C_2 > 0$ are universal constants. Also for each edge of the ReLU NN, the corresponding weight is either independent of $g$, or is of the form $g_{i}(\vy)$ for some fixed $\vy \in [-M,M]^d$ and coordinate $i = 1,\ldots,p$.
\end{cor}


\noindent \textbf{Convolutional Neural Network:} As many successful network applications on inverse problems result from the use of filters in the CNN architectures  \citep{jin2017deep}, we are particularly interested in the expressive power of CNN in approximating the H\"{o}lder functions. Currently, known non-asymptotic results for CNNs include  \citep{zhou2020universality,petersen2020equivalence,yarotsky2022universal}, but they are all established under stricter assumptions about $f$ than mere Lipschitz continuity.  On the other hand, the ResNet-based CNN with the following architecture has been shown to possess a good convergence rate for even H\"older continuous functions: 
\begin{equation}\label{eq:recnn} CNN^{\sigma}_{\theta} := FC_{W,b}\circ(\textrm{Conv}^{\sigma}_{{\bf \omega_M},{\bf b_M}}+\textrm{id})\circ \cdots \circ(\textrm{Conv}^{\sigma}_{{\bf \omega_1},{\bf b_1}}+\textrm{id})\circ P, 
\end{equation}
where $\sigma$ is the activation function, each $\textrm{Conv}_{{\bf \omega_m},{\bf b_m}}$ is a convolution layer with $L_m$ filters $\omega_m^{1}$,...,$\omega_m^{L_m}$ stored in ${\bf \omega_M}$ and $L_m$ bias $b_m^1,...,b_m^{L_m}$ stored in $\bf b_m$. The addition by the identity map, $\textrm{Conv}^{\sigma}_{{\bf \omega_M},{\bf b_M}}+\textrm{id}$, makes it a residual block. Here $FC_{W,b}$ represents a fully connected layer appended to the final layer of the network, and $P: \mathbb{R}^d \rightarrow \mathbb{R}^{d\times C}:  x \rightarrow (x, 0,\cdots 0)$ is a padding operation that adds zeros to align the number of channels in the first and second layers. One can see that the ResNet-based CNN is essentially a normal CNN with skip connections.

The following asymptotic result is proven in  \citep{oono2019approximation}.

\begin{prop}[Corollary 4 from \citep{oono2019approximation}] 
\label{prop:HolderCNNlowdim}
Let $f : [-1,1]^d \to \R$ be an $\alpha$-H\"{o}lder function. Then, for any $K \in \{2, . . . , d\}$, there exists a CNN $f^{\text{(CNN)}}$ with $O(N)$ residual blocks, each of which has depth $O(\log N)$ and $O(1)$ channels, and whose filter size is at most $K$, such that $\|f-f^{\text{(CNN)}}\|_{\infty} \le \widetilde{O}(N^{-\alpha/d})$, where the $\widetilde{O}$ denotes that $\log N$ factors have been suppressed.
\end{prop}

\subsection{Our Main results}
\label{sec:MainResults}

We can now pair results guaranteeing the existence of a $\rho$-JL embedding of $\calS \subset \R^D$ into $[-M,M]^d$ with results for approximating functions on $[-M,M]^d$ to obtain results for approximating functions on $\calS \subset \R^D$.

By combining Proposition~\ref{prop:ContinuousFNNlowdim} with Propositions~\ref{prop:JLCoveringNumber}(a) and \ref{prop:JLGaussianWidth}, we obtain the following result for approximating continuous functions on a high-dimensional set by a feedforward ReLU NN. The proof of this theorem is given in Appendix~\ref{sec:ProofContinuousFNN}.

\begin{thm}
\label{thm:ContinuousFNN}
Let $d < D$ be positive integers, and let $\rho \in (0,\tfrac{1}{2})$ be a constant. Let $\Delta(r)$ be a modulus of continuity. Let $\calS \subset \R^D$ be a bounded set and $U_\mathcal{S}$ be its set of unit secants. Suppose that
   $$d \gtrsim \min\left\{ \rho^{-2}\log  \mathcal{N}(U_{\mathcal{S}}, \|\cdot\|_2, \tfrac{\rho}{4\sqrt{3D}}),\ \ \rho^{-2}\left(\omega(U_{\calS})\right)^2\right\},$$
  where $\mathcal{N}(U_{\mathcal{S}}, \|\cdot\|_2, \frac{\rho}{4\sqrt{3D}})$ is the covering number and  $\omega (U_{\calS})$ is the Gaussian width of $U_{\calS}$.
 Then, for any positive integer $N$, there exists a ReLU neural network architecture with at most $$(p+C_1)(2N+1)^d + Dd \ \text{edges},$$ $$C_2(2N+1)^d + p + D \ \text{nodes},$$ $$\text{and} \ \ceil{\log_2(d+1)} + 3 \ \text{layers}$$ that can $\sqrt{p}\Delta(\tfrac{M\sqrt{d}}{(1-\rho)N})$-approximate the class of $\calC(\Delta(r))$-functions $f : \calS \to \R^p$, where $M = \sup_{\vx \in \calS}\|\mA\vx\|_{\infty}$.
\end{thm}

By applying Corollary~\ref{cor:HolderFNNlowdim} instead of Proposition~\ref{prop:ContinuousFNNlowdim}, we obtain a variant of the previous theorem for approximating H\"{o}lder functions on a high-dimensional set by a feedforward ReLU NN. The proof of this theorem is given in Appendix~\ref{sec:ProofHolderFNN}.

\begin{thm}
\label{thm:HolderFNN}
Let $d < D$ be positive integers, and let $L > 0$, $\alpha \in (0,1]$, and $\rho \in (0,\tfrac{1}{2})$ be constants. Let $\calS \subset \R^D$ be a bounded set and $U_\mathcal{S}$ be its set of unit secants. Suppose that
   $$d \gtrsim \min\left\{ \rho^{-2}\log  \mathcal{N}(U_{\mathcal{S}}, \|\cdot\|_2, \tfrac{\rho}{4\sqrt{3D}}),\ \ \rho^{-2}\left(\omega(U_{\calS})\right)^2\right\},$$
  where $\mathcal{N}(U_{\mathcal{S}}, \|\cdot\|_2, \frac{\rho}{4\sqrt{3D}})$ is the covering number and  $\omega (U_{\calS})$ is the Gaussian width of $U_{\calS}$.
 Then, there exists a ReLU neural network architecture with at most $$(p+C_1)\left(2\ceil{\dfrac{M\sqrt{d}}{(1-\rho)(\eps/L)^{1/\alpha}}}+1\right)^d + Dd \ \text{edges},$$ $$C_2\left(2\ceil{\dfrac{M\sqrt{d}}{(1-\rho)(\eps/L)^{1/\alpha}}}+1\right)^d + p + D \ \text{nodes},$$ $$\text{and} \ \ceil{\log_2(d+1)} + 3 \ \text{layers}$$ that can $\eps$-approximate the class of $(L,\alpha)$-H\"{o}lder functions $f : \calS \to \R^p$, where $M = \sup_{\vx \in \calS}\|\mA\vx\|_{\infty}$.
\end{thm}

By combining Proposition~\ref{prop:HolderCNNlowdim} with Proposition~\ref{prop:JLCoveringNumber}(b), we obtain the following result for approximating H\"{o}lder functions on a high-dimensional set by a ResNet type CNN. The proof of this theorem is given in Appendix~\ref{sec:ProofHolderCNN}.

\begin{thm}
\label{thm:HolderCNN}
Let $d < D$ be positive integers, and let $\alpha \in (0,1]$ and $\rho \in (0,\tfrac{1}{2})$ be constants. Let $\calS \subset \R^D$ be a bounded set and $U_\mathcal{S}$ be its set of unit secants. Suppose that
   $$d \gtrsim \rho^{-2}\log(4D+4d)\log \mathcal{N}(U_{\mathcal{S}}, \|\cdot\|_2, \tfrac{\rho}{4\sqrt{3D}}).$$  Then, for any $\alpha$-H\"{o}lder function $f : \calS \to \R^p$, there exists a ResNet type CNN $f^{\text{(CNN)}}$ in the form of \eqref{eq:recnn} with $O(N)$ residual blocks, each of which has a depth $O(\log N)$ and $O(1)$ channels, and whose filter size is at most $K$ such that $\|f-f^{\text{(CNN)}}\|_{\infty} \le \widetilde{O}(N^{-\alpha/d})$, where the $\widetilde{O}$ denotes that $\log N$ factors have been suppressed.
\end{thm}

\section{Applications to Inverse Problems}
Now we focus on inverse problems and demonstrate how the main theorems can be used to provide a reasonable estimate of the size of the neural networks needed to solve some classical inverse problems in signal processing. The problems we consider here are sparse recovery, blind deconvolution, and matrix completion. 
In all these inverse problems, we want to recover some  signal  $\vx\in \mathcal{S}$ from its forward measurement $\vy=F(\vx)$,  where the forward map $F$ is assumed to be known. The minimal assumption we have to impose on $F$ is that it has an inverse with a known modulus of continuity.  As we shall see, however, in the following examples we can in fact assume Lipschitz continuity of the inverse.\\

\noindent \textbf{Assumption 1} (invertibility of the forward map): Let $\mathcal{S}$ be the domain of the forward map $F$, and $\mathcal{Y}=F(\mathcal{S})$ be the range. Assume that the inverse operator $F^{-1}: \mathcal{Y}\rightarrow \mathcal{S}$ exists and is Lipschitz continuous with constant $L$, so that
\[
\|F^{-1}(\vy_1)-F^{-1}(\vy_2)\|_2 \leq L\|\vy_1-\vy_2\|_2, \quad \text{for all} \quad \vy_1,\vy_2\in \mathcal{Y}.
\]

For any inverse problems satisfying Assumption 1, Theorems~\ref{thm:HolderFNN} and \ref{thm:HolderCNN} provide ways to estimate  the size of the universal approximator networks for the inverse map. When applying the theorems to each problem, we need to first estimate the covering number of $U_{\mathcal{Y}}$.  Depending on the problem, one may estimate the covering number either numerically or theoretically. If the domain $\mathcal{Y}$ of the inverse map is irregular and discrete, then it may be easier to compute the covering number numerically. If the domain has a nice mathematical structure, then we may be able to estimate it theoretically. In the three examples below we will use theoretical estimation. From them, we see that it is quite common for inverse problems to have a small intrinsic complexity.  As a result, Theorems~\ref{thm:HolderFNN} and \ref{thm:HolderCNN} can significantly reduce the required size of the network below the sizes provided by previously known results.



We begin by emphasizing that the covering number that the theorems use is for the unit secants of $\mathcal{Y}$, which can be appreciably larger than the covering number of $\mathcal{Y}$ itself. \\

\noindent \textbf{Sparse recovery:} Sparsity is now one of the most commonly used priors in inverse problems as signals in many real applications possess a certain level of sparsity with respect to a given basis. For simplicity, we consider strictly sparse signals that have a small number of nonzero entries.
Let $\Sigma_s^N$ be the set of $s-$sparse vectors of length $N$.  Assume a sparse vector is measured linearly $\vy=\mPhi\vx\equiv F(\vx)$, where $\mPhi \in \R^{m \times N}$. The inverse problem amounts to recovering $\vx$ from $\vy$. Now that we want to use a network to approximate the inverse map $F^{-1}: \mPhi(\Sigma_s^N) \to \Sigma_s^N$, and estimate the size of the network through the theorems, we need to estimate the covering number of the unit secant $U_{\mPhi(\Sigma_s^N)}$.


\begin{prop}
\label{prop:SparseRecovery} 
Suppose $\mPhi$ satisfies the restricted isometry property so that the inverse  map $F^{-1}: \mPhi(\Sigma_s^N) \to \Sigma_s^N$ is Lipschitz continuous with Lipschitz constant $L$. Then, for any $\rho \in (0,\tfrac{1}{2})$, there exists a $\rho$-JL embedding $\mA \in \R^{d \times m}$ of the set of possible observations $$\mathcal{Y} = \{ \vy = \mPhi\vx : \vx \in \Sigma^N_s\},$$ into $\R^d$, provided that $$d \gtrsim \rho^{-2}s\log\dfrac{N\sqrt{m}}{s\rho}.$$ For such choice of $d$, we have that for any bounded subset $\calS \subset \mathcal{Y}$, there exists a ReLU neural network architecture with at most $$(N+C_1)\left(2\ceil{\dfrac{LM\sqrt{d}}{(1-\rho)\eps}}+1\right)^d + md \ \text{edges},$$ $$C_2\left(2\ceil{\dfrac{LM\sqrt{d}}{(1-\rho)\eps}}+1\right)^d + m + N \ \text{nodes},$$ $$\text{and} \ \ceil{\log_2(d+1)} + 3 \ \text{layers}$$ which can $\eps$-approximate the inverse map $F^{-1}$ over the set $\calS$. Here $M = \max_{\vy \in \calS}\|\mA\vy\|_2$.
\end{prop}


\noindent \textbf{Blind deconvolution:} Blind deconvolution concerns the recovery of a signal $\vx$ from its blurry measurements
\begin{equation}\label{eq:decon}
\vy = \vk \otimes \vx 
\end{equation}
when the kernel $\vk$ is also unknown. Here $\otimes$ denotes the convolution operator.
Note that blind-deconvolution is an ill-posed problem due to the existence of a scaling ambiguity between $\vx$ and $\vk$, namely, if $(\vk,\vx)$ is a solution, then $(\alpha \vk, \frac{1}{\alpha} \vx)$ with $\alpha \neq 0$ is also a solution. To resolve this issue,  we focus on recovering the outer product $\vx\vk^T$, where $\vx$ and $\vk$ here are both  column vectors. 

The recovery of the outer product $\vx\vk^T$ from the convolution $\vy=\vk \otimes \vx $ can be well-posed in various settings \citep{Lee15,Ahmed13}. For example, \citep{Ahmed13} showed that if we assume $\vx=\mPhi \vu$ and $\vk = \mPsi\vv$, where $\mPhi \in \mathbb{R}^{N \times n} (n<N)$ is i.i.d. Gaussian matrix and $\mPsi \in \mathbb{R}^{N \times m} (m<N)$ is a matrix of small coherence, then for large enough $N$,  the outer-product $\vx\vk^T$ can be stably recovered from $\vy$ in the following sense. For any two signal-kernel pairs $(\vx,\vk)$, $(\tilde{\vx},\tilde{\vk})$ and their corresponding convolutions $\vy$, $\tilde{\vy}$, we have
\begin{equation}\label{eq:lip}
\left\|\vx\vk^T - \tilde{\vx}\tilde{\vk}^T\right\| \leq L\|\vy - \tilde{\vy}\|_2
\end{equation}
for some $L$.
To determine the size of a neural network to approximate the inverse map by $F^{-1}: \vy \rightarrow \vx\vk^T$, we need to estimate the covering number of the unit secant cone of $\mathcal{Y} = \{ \vy = \vx \otimes \vk, \vx \in \mPhi \Sigma_s^N, \vk\in \text{span}{\mPsi}\}  $, which is done in the following proposition. The proof of this proposition is given in Appendix ~\ref{sec:ProofBlindDeconvolution}.


\begin{prop}\label{prop:BlindDeconvolution} Suppose the inverse  map $F^{-1}: \vy \rightarrow \vx\vk^T$ is Lipschitz continuous with Lipschitz constant $L$. Then, for any $\rho \in (0,\tfrac{1}{2})$, there exists a $\rho$-JL embedding $\mA \in \R^{d \times N}$ of the set of possible observations $$\mathcal{Y} = \{ \vy = \vx \otimes \vk, \vx \in \text{span}(\mPhi), \vk\in \text{span}({\mPsi})\},$$ into $\R^d$, provided that $$d \gtrsim \rho^{-2}\max\{m,n\}\log\dfrac{L\sqrt{N}}{\rho}.$$ For such choice of $d$, we have that for any bounded subset $\calS \subset \mathcal{Y}$, there exists a ReLU neural network architecture with at most $$(N^2+C_1)\left(2\ceil{\dfrac{LM\sqrt{d}}{(1-\rho)\eps}}+1\right)^d + Nd \ \text{edges},$$ $$C_2\left(2\ceil{\dfrac{LM\sqrt{d}}{(1-\rho)\eps}}+1\right)^d + N^2 + N \ \text{nodes},$$ $$\text{and} \ \ceil{\log_2(d+1)} + 3 \ \text{layers}$$ which can $\eps$-approximate the inverse map $F^{-1}$ over the set $\calS$. Here $M = \max_{\vy \in \calS}\|\mA\vy\|_2$
\end{prop}

\noindent \textbf{Matrix completion:}
Matrix Completion is a central task in machine learning where we want to recover a matrix from its partially observed entries. It arises from a number of applications including image super resolution \citep{shi2013low,cao2014image}, image/video denoising \citep{ji2010robust}, recommender systems \citep{zheng2016neural,monti2017geometric}, and gene-expression prediction \citep{kapur2016gene}, etc.. Recently neural network models  have achieved state-of-the-art
performance  \citep{zheng2016neural,monti2017geometric,dziugaite2015neural,he2017neural}, but a general existence result in the non-asymptotic regime is still missing. 

In this setting, the measurements $\mY=P_{\Omega}(\mX)$ consists of a set of observed entries of the unknown low-rank matrix $\mX$, where $\Omega$ is the index set of the observed entries and $P_\Omega$ is the mask that sets all but entries in $\Omega$ to 0. Let $M_r^{n,m}$ be the set of $n\times m$ matrices with rank at most $r$, and $\mX\in M_r^{n,m}$. If the mask is random, and the left and right eigenvectors $\mU, \mV$ of $\mX$ are incoherent in the sense that 
\begin{align}\label{eq:incoh}
\max_{1\leq i\leq n}& \left\|\mU^T\ve_i\right\|_2 \leq \sqrt{\frac{ \mu_0 r}{n}}, \quad \max_{1\leq i\leq m} \left\|\mV^T\ve_i\right\|_2 \leq \sqrt{\frac{ \mu_0 r}{m}}, ~{\rm and}\\ &  \max_{1\leq i\leq n,1\leq j\leq m} \left|(\mU\mV^T)_{i,j}\right| \leq \sqrt{\frac{ \mu_1 r}{nm}} \notag 
\end{align}
all hold, then it is known (see, e.g., \citep{Candes10}) that the inverse map $F^{-1}: \mY \rightarrow \mX$ exists and is Lipschitz continuous with overwhelming probability provided that the number of observations $$|\Omega| \gtrsim \mu_0r\max\{m,n\}\log^2 \max\{m,n\}.$$
Let us denote the set of low-rank matrices satisfying \eqref{eq:incoh} by $\mathcal{C}$. To estimate the complexity of the inverse map, we wiil now compute the covering numbers of $U_{\mathcal{Y}}$ for $\mathcal{Y}=\{\mY=P_{\Omega}(\mX): \mX \in M^{m,n}_r \cap \mathcal{C}\}$.  This is done using the following proposition, which is proven in Appendix~\ref{sec:ProofMatrixCompletion}.


\begin{prop}\label{prop:MatrixCompletion}
Suppose the mask is chosen so that the inverse  map $F^{-1}: \mY=P_{\Omega}(\mX) \rightarrow \mX$ is Lipschitz continuous with Lipschitz constant $L$. Then, for any $\rho \in (0,\tfrac{1}{2})$, there exists a $\rho$-JL embedding $A : \R^{m \times n} \to \R^d$ of the set of possible observations $$\mathcal{Y} = \{P_{\Omega}(\mX): \mX \in M^{m,n}_r \cap \mathcal{C}\},$$ into $\R^d$ provided that $$d \gtrsim \rho^{-2}r(m+n)\log\dfrac{L\sqrt{mn}}{\rho}.$$ For such choice of $d$, we have that for any bounded subset $\calS \subset \mathcal{Y}$, there exists a ReLU neural network architecture with at most $$(mn+C_1)\left(2\ceil{\dfrac{LM\sqrt{d}}{(1-\rho)\eps}}+1\right)^d + mnd \ \text{edges},$$ $$C_2\left(2\ceil{\dfrac{LM\sqrt{d}}{(1-\rho)\eps}}+1\right)^d + 2mn \ \text{nodes},$$ $$\text{and} \ \ceil{\log_2(d+1)} + 3 \ \text{layers}$$ which can $\eps$-approximate the inverse map $F^{-1}$ over the set $\calS$. Here $M = \max_{\mY \in \calS}\|\text{vec}(A(\mY))\|_{\infty}$
\end{prop}

\begin{remark}
In each of Propositions~\ref{prop:SparseRecovery}, \ref{prop:BlindDeconvolution}, and \ref{prop:MatrixCompletion}, it is shown that the number of nodes and edges in a neural network to solve a matrix completion problem scale (with respect to $\eps$) exponentially with the intrinsic dimension of the problem $d$ and not the much larger ambient dimension.
\end{remark}
\begin{remark}
One can also get bounds on the size of a ResNet CNN needed for each of these inverse problems by enlarging $d$ by a factor of the logarithm of the ambient dimension (which ensures a circulant $\rho$-JL embedding exists), and then applying Theorem~\ref{thm:HolderCNN} instead of Theorem~\ref{thm:HolderFNN}.
\end{remark}

\section{Conclusions, Limitations, and Discussion} 
The main message of this paper is that when neural networks are used to approximate H\"{o}lder continuous functions, the size of the network only needs to grow exponentially with respect to the intrinsic complexity of the input set measured using either its Gaussian width or its covering numbers.  Therefore, it is often more optimistic than previous estimates that require the size of the network to grow exponentially with respect to the extrinsic input dimension. 

We note that when the domain of the input is a manifold, our techniques would yield results that are slightly worse than optimal. Specifically, \cite{baraniuk2009random,Iwen22} both show that a $k$-dimensional manifold has a $\rho$-JL embedding into $d = \widetilde{O}(k/\rho^2)$ dimensions. As such, our results for the size of a neural network that approximates $\alpha$-H\"{o}lder functions on a manifold would scale with respect to $\eps$ like $O(\eps^{-d/\alpha})$ where $d = \widetilde{O}(k/\rho^2)$ as compared to the $O(\eps^{-k/\alpha}\log \tfrac{1}{\eps})$ scaling in \citep{chen2019efficient}. This is of course to be expected as the neural network in \citep{chen2019efficient} is constructed carefully with respect to the structure of the manifold, while our construction is more oblivious to the domain of the function. In addition, as explained in Section \ref{sec:Example}, our result only holds for H\"{o}lder indices $0 < \alpha \le 1$ as opposed to all $\alpha>0$ in \citep{chen2019efficient}. However, the use of the JL embedding allows our result to be stated under a more general (non-manifold) assumption of the input set and for a much broader class of neural networks -- although we only stated it for feedforward neural networks and the ResNet type of convolutional neural networks, the same idea naturally applies to other types of networks as long as an associated JL-map exists. 

The estimates the results herein provide for the network size ultimately depend on the complexity of the input set, measured by either the covering numbers, or by the Gaussian width, of its set of unit secants.  The computation of these quantities varies case by case, and in some cases might be rather difficult. This is a possible limitation of the proposed method. In particular, if the estimation of the input set complexity is not tight enough, the results herein may again become overly pessimistic. Having said that, for many classical inverse problems, the covering numbers and the Gaussian width estimates are not too difficult to calculate. As we demonstrated in Section 4, there are many known properties that one can use to facilitate the calculation. And, when a training dataset is given, one can even approximate covering numbers numerically with off-the-shelf algorithms. 

Finally, although the applications of neural networks to inverse problems are seeing a lot of current success, there are also failed attempts that don't work for unknown reasons.  One common explanation is that the size of the network in use is not large enough for the targeted application.  Since inverse problems models usually have a much higher intrinsic dimensionality than, say, image classification models, the required network sizes might indeed be much larger. Classical universal approximation theorems only guarantee small errors when the network size approaches infinity, therefore are not very helpful in the non-asymptotic regime where one has to choose the network size.  And, such parameter choices are now known to be critical to good performance. We hope the presented results provide more insight in this regard.   

\section*{Acknowledgments}
Rongrong Wang was supported in part by NSF CCF-2212065. Mark Iwen was supported in part by NSF DMS 2106472. We would also like to note that we became aware of similar independent work by Demetrio Labate and Ji Shi \citep{LabateShi23} during a poster session at Texas A\&M as part of the Inaugural CAMDA Conference (May 22 - 25, 2023) after the initial submission of this paper. We regard this independent and parallel development of similar results and proof strategies to be an indication of their timeliness, utility, and general interest.

\bibliographystyle{unsrt}
\bibliography{ref}
\vspace{\fill}
\pagebreak

\appendix

\section{Proof of Theorems~\ref{thm:MainContinuousFNN} and ~\ref{thm:MainContinuousCNN}}
\label{sec:ProofMainContinuous}
First, we prove the following lemma.

\begin{lem}
\label{lem:JLfunctionContinuous}
Let $d < D$ and $p$ be positive integers, and let $M > 0$ and $\rho \in (0,1)$ be constants. Let $\calS \subset \R^D$ be a bounded subset for which there exists a $\rho$-JL embedding $\mA \in \R^{d \times D}$ of $\calS$ into $[-M,M]^d$. Let $\Delta(r)$ be a modulus of continuity. Then, for any function $f : \calS \to \R^p$ that admits $\Delta(r)$ as a modulus of continuity, there exists a function $g : [-M,M]^d \to \R^p$ that admits $\sqrt{p}\Delta(\tfrac{r}{1-\rho})$ as a modulus of continuity such that $g(\mA\vx) = f(\vx)$ for all $\vx \in \calS$.
\end{lem}

\begin{proof}
For any $\vx,\vx' \in \calS$, if $\mA\vx = \mA\vx'$ then since $\mA$ is a $\rho$-JL embedding of $\calS$, we have that $\|\vx-\vx'\|_2 \le \tfrac{1}{1-\rho}\|\mA\vx-\mA\vx'\|_2 =\tfrac{1}{1-\rho}\|\vct{0}\|_2  = 0$, and so, $\|\vx-\vx'\|_2 = 0$, i.e., $\vx = \vx'$. Therefore, the map $\vx \mapsto \mA\vx$ from $\calS$ to $\mA(\calS) := \{\mA\vx : \vx \in \calS\}$ is invertible. We define $A^{-1} : \mA(\calS) \to \calS$ to be the inverse of the map $\vx \mapsto \mA\vx$.

Now, for any function $f : \calS \to \R^p$ which admits $\Delta(r)$ as a modulus of continuity, we define $\widetilde{g} : \mA(\calS) \to \R^p$ by $\widetilde{g} = f \circ A^{-1}$. Then, for any $\vy,\vy' \in \mA(\calS)$, we have 
\begin{align*}
\|\widetilde{g}(\vy)-\widetilde{g}(\vy')\|_2 &= \left\|f(A^{-1}(\vy))-f(A^{-1}(\vy'))\right\|_2 & \text{since} \ g = f \circ A^{-1}
\\
& \le \Delta\left(\left\|A^{-1}(\vy) - A^{-1}(\vy')\right\|_2\right) & \text{since} \ f \ \text{admits} \ \Delta(r) \ \text{as a M.O.C.}
\\
&\le \Delta\left(\tfrac{1}{1-\rho}\left\|\mA A^{-1}(\vy) - \mA A^{-1}(\vy')\right\|_2\right) & \text{since} \ \mA \ \text{is a} \ \rho\text{-JL embedding of} \ \calS
\\
&= \Delta\left(\tfrac{1}{1-\rho}\left\|\vy - \vy'\right\|_2\right). & \text{since} \ A^{-1} \ \text{is the inverse of} \ \vx \mapsto \mA\vx
\end{align*}

Therefore, $\widetilde{g} : \mA(\calS) \to \R^p$ admits $\Delta(\tfrac{r}{1-\rho})$ as a modulus of continuity. Then, since $\mA(\calS) \subset [-M,M]^d$, we can extend $\widetilde{g}$ to a function $g : [-M,M]^d \to \R^p$ via the definition $$g_i(\vy) := \inf_{\vz \in \mA(\calS)}\left[\widetilde{g}_i(\vz) + \Delta\left(\tfrac{1}{1-\rho}\|\vy-\vz\|_2\right)\right] \quad \text{for} \quad i = 1,\ldots,p.$$ Since $\widetilde{g}$ admits a modulus of continuity of $\Delta(\tfrac{r}{1-\rho})$, so does each coordinate $\widetilde{g}_i$. The above extension formula preserves the modulus of continuity of each coordinate $g_i$, and so $g$ admits modulus of continuity $\sqrt{p}\Delta(\tfrac{r}{1-\rho})$. Also, this extension satisfies $g(\vy) = \widetilde{g}(\vy)$ for all $\vy \in \mA(\calS)$. Finally, for any $\vx \in \calS$, we have $\mA\vx \in \mA(\calS)$, and so, $g(\mA\vx) = \widetilde{g}(\mA\vx) = f(A^{-1}(\mA\vx)) = f(\vx)$, as required.
\end{proof}

With Lemma~\ref{lem:JLfunctionContinuous}, we can now prove each of the parts of Theorems~\ref{thm:MainContinuousFNN} and \ref{thm:MainContinuousCNN}. As a reminder, we assume that $\calS \subset \R^D$ is a bounded set for which there exists a $\rho$-JL embedding $\mA \in \R^{d \times D}$ of $\calS$ into $[-M,M]^d$.

1a) Let $f: \calS \to \R^p$ be a function that admits $\Delta(r)$ as a modulus of continuity. By Lemma~\ref{lem:JLfunctionContinuous}, there exists a function $g : [-M,M]^d \to \R^p$ that admits $\sqrt{p}\Delta(\tfrac{r}{1-\rho})$ as a modulus of continuity such that $f(\vx) = g(\mA\vx)$ for all $\vx \in \calS$. By assumption, $g$ can be $\eps$-approximated by a feedforward neural network with at most $\calN$ nodes, $\calE$ edges, and $\calL$ layers. In other words, there exists a function $\widehat{g}$ such that $\|\widehat{g}(\vy)-g(\vy)\|_{\infty} \le \eps$ for all $\vy \in [-M,M]^d$, and $\widehat{g}$ can be implemented by a feedforward neural network with at most $\calN$ nodes, $\calE$ edges, and $\calL$ layers. 

Define another function $\widehat{f} = \widehat{g} \circ \mA$, i.e., $\widehat{f}(\vx) = \widehat{g}(\mA\vx)$ for all $\vx \in \calS$. Since $\mA(\calS) \subset [-M,M]^d$ by assumption, we have that $\mA\vx \in [-M,M]^d$ for all $\vx \in \calS$. Then, $\|\widehat{f}(\vx)-f(\vx)\|_{\infty} = \|\widehat{g}(\mA\vx)-g(\mA\vx)\|_{\infty} \le \eps$ for all $\vx \in \calS$, i.e., $\widehat{f}$ is an $\eps$-approximation of $f$. 

Furthermore, we can construct a feedforward neural network to implement $\widehat{f} = \widehat{g} \circ \mA$ by having a linear layer to implement the map $\vx \mapsto \mA\vx$, and then feeding this into the neural network implementation of $\widehat{g}$. The map $\vx \mapsto \mA\vx$ can be implemented with $D$ nodes for the input layer, and $Dd$ edges between the input nodes and the first hidden layer. By assumption, $\widehat{g}$ can be implemented by a feedforward neural network with at most $\calN$ nodes, $\calE$ edges, and $\calL$ layers. Hence, $\widehat{f} = \widehat{g} \circ \mA$ can be implemented by a feedforward neural network with at most $\calN+D$ nodes, $\calE+Dd$ edges, and $\calL+1$ layers, as desired. 

1b) If the same feedforward neural network architecture $\eps$-approximates every function $g : [-M,M]^d \mapsto \R^p$ that admits $\sqrt{p}\Delta(\tfrac{r}{1-\rho})$ as a modulus of continuity, then our construction of a feedforward neural network that implements $\widehat{f} = \widehat{g} \circ \mA$ has the same architecture for every function $f : \calS \to \R^p$ that admits $\Delta(r)$ as a modulus of continuity. Hence, the same bounds on the number of nodes, edges, and layers hold.

2a) In a similar manner as 1a), we form a CNN that can approximate $f = g \circ \mA$ by first implementing the linear map $\vx \mapsto \mA\vx$ with a CNN and feeding this into a CNN that approximates $g$.

The JL matrix $\mA=\mM\mD$ can be represented by a Resnet-CNN structure as follows. Let $\vx$ be the input of the network, then $\mD\vx$, the random sign flip of the input can be realized by setting the weight/kernel $w_1$ of the first two layers to be the delta function, and the bias vectors to take large values at the location where $\mD$ has a $1$, and small values where $\mD$ has a $-1$. Then with the help of the ReLU activation, we can successfully flip the signs. More explicitly, set $T = \sup_{\vx \in \calS}\|\vx\|_{\infty}$ so that $\vx\in [-T,T]^D$ for all $\vx \in \calS$. Let $\vb_i$ be the bias to be added to the $i$th coordinate of the input. We design a 2 layer Resnet-CNN, $\ell(\vx)$, as follows
\[
\ell(\vx)_i = \text{ReLU}(2\vx_i+\vb_i)-\text{ReLU}(\vb_i)-\vx_i,  \quad i=1,...,D .
\]
The bias $\vb_i$ is chosen to realize the sign flip as follows.
If $\mD_{ii}$ contains a $1$, then we set $\vb_i = 2T$, which will make $\ell(\vx)_i = \vx_i$. If $\mD_{ii}$ contains a $-1$, then we set $\vb_i = -2T$, which will make $\ell(\vx)_i = -\vx_i$, thus realizing the sign-flip. A similar architecture can also be used to realize the application of $\mM$ to $\mD\vx$, which is a convolution followed by a mask (i.e., setting certain entries of $\vx$ to 0). Specifically, we let $\vm \in \R^D$ denote the first row of the partial circulant matrix $\mM$, and we set $T' = \sup_{\vx \in \calS}\|\vm \otimes \mD\vx\|_{\infty}$ so that $\vm \otimes \mD\vx \in [-T',T']^D$. Then, we let $\vb' \in \R^D$ be a bias vector whose first $d$ entries are $T'$ and whose last $D-d$ entries are $-T'$. Then, we pass $\mD\vx = \ell(\vx)$ through the following Resnet-CNN $$\ell'(\mD\vx) = \text{ReLU}(\vm \otimes \mD\vx+\vb')-\text{ReLU}(\vb').$$ This convolves $\mD\vx$ with $\vm$ and then sets the last $D-d$ entries to $0$.

This Resnet-CNN that implements $\vx \mapsto \mA\vx$ requires $2D$ nodes, $D$ parameters, and $2$ layers to apply $\mD$ to $\vx$, and an additional $2D$ nodes, $D$ parameters, and $2$ layers to apply $\mM$ to $\mD\vx$. By adding this to the $\calN$ nodes, $\calP$ parameters, and $\calL$ layers needed for a CNN to approximate the function $g : [-M,M]^d \to \R^p$ that admits $\sqrt{p}\Delta(\tfrac{r}{1-\rho})$ as a modulus of continuity, we obtain that the function $f : \calS \to \R^p$ that admits $\Delta(r)$ as a modulus of continuity can be approximated by a Resnet-CNN with $\calN+4D$ nodes, $\calP+2D$ parameters, and $\calL+4$ layers.

2b) If the same convolutional neural network architecture $\eps$-approximates every function $g : [-M,M]^d \mapsto \R^p$ that admits $\sqrt{p}\Delta(\tfrac{r}{1-\rho})$ as a modulus of continuity, then our construction of a convolutional neural network that implements $\widehat{f} = \widehat{g} \circ \mA$ has the same architecture for every function $f : \calS \to \R^p$ that admits $\Delta(r)$ as a modulus of continuity. Hence, the same bounds on the number of nodes, parameters, and layers hold.

\section{Proofs of Theorems~\ref{thm:MainHolderFNN} and \ref{thm:MainHolderCNN}}
\label{sec:ProofMainHolder}

We can strengthen Lemma~\ref{lem:JLfunctionContinuous} for the special case where $\Delta(r) = Lr^{\alpha}$ for some constants $L > 0$ and $\alpha \in (0,1]$, ($f$ is $(L,\alpha)$-H\"{o}lder) in a way that removes the $\sqrt{p}$ factor. This is done by using a theorem which allows us to extend an $\mathbb{R}^p$-valued H\"{o}lder function instead of extending each coordinate separately.

\begin{lem}
\label{lem:JLfunctionHolder}
Let $d < D$ and $p$ be positive integers, and let $L,M > 0$, $\alpha \in (0,1]$, and $\rho \in (0,1)$ be constants. Let $\calS \subset \R^D$ be a bounded subset for which there exists a $\rho$-JL embedding $\mA \in \R^{d \times D}$ of $\calS$ into $[-M,M]^d$. Then, for any $(L,\alpha)$-H\"{o}lder function $f : \calS \to \R^p$, there exists an $(\tfrac{L}{(1-\rho)^{\alpha}},\alpha)$-H\"{o}lder function $g : [-M,M]^d \to \R^p$ such that $g(\mA\vx) = f(\vx)$ for all $\vx \in \calS$.
\end{lem}

\begin{proof}
For any $(L,\alpha)$-H\"{o}lder function $f : \calS \to \R^p$, we can use the same argument as in the proof of Lemma~\ref{lem:JLfunctionContinuous} to show that the function $\widetilde{g} : \mA(\calS) \to \R^p$ defined by $\widetilde{g} = f \circ A^{-1}$ satisfies 
\begin{align*}
\|\widetilde{g}(\vy)-\widetilde{g}(\vy')\|_2 &= \left\|f(A^{-1}(\vy))-f(A^{-1}(\vy'))\right\|_2 & \text{since} \ g = f \circ A^{-1}
\\
& \le L\left\|A^{-1}(\vy) - A^{-1}(\vy')\right\|_2^{\alpha} & \text{since} \ f \ \text{is} \ (L,\alpha)\text{-H\"{o}lder}
\\
&\le \dfrac{L}{(1-\rho)^{\alpha}}\left\|\mA A^{-1}(\vy) - \mA A^{-1}(\vy')\right\|_2^{\alpha} & \text{since} \ \mA \ \text{is a} \ \rho\text{-JL embedding of} \ \calS
\\
&= \dfrac{L}{(1-\rho)^{\alpha}}\left\|\vy - \vy'\right\|_2^{\alpha}. & \text{since} \ A^{-1} \ \text{is the inverse of} \ \vx \mapsto \mA\vx
\end{align*}
for all $\vy, \vy' \in \mA(\calS)$. Therefore, $\widetilde{g} : \mA(\calS) \to \R^p$ is $(\tfrac{L}{(1-\rho)^{\alpha}},\alpha)$-H\"{o}lder. Then, since $\mA(\calS) \subset [-M,M]^d$, by Theorem 1(ii) in \citep{minty1970extension} (which is a generalization of the Kirszbraun theorem \citep{schwartz1969nonlinear}), there exists a $(\tfrac{L}{(1-\rho)^{\alpha}},\alpha)$-H\"{o}lder extension of $\widetilde{g}$ to $[-M,M]^d$, i.e., a function $g : [-M,M]^d \to \R^p$ which is $(\tfrac{L}{(1-\rho)^{\alpha}},\alpha)$-H\"{o}lder on $[-M,M]^d$ and satisfies $g(\vy) = \widetilde{g}(\vy)$ for all $\vy \in \calS$. Finally, for any $\vx \in \calS$, we have $\mA \vx \in \mA( \mathcal{S} )$, and so, $g(\mA\vx) = \widetilde{g}(\mA\vx) = f(A^{-1}(\mA\vx)) = f(\vx)$, as required.
\end{proof}

The proofs of each of the parts of Theorems~\ref{thm:MainHolderFNN} and \ref{thm:MainHolderCNN} are identical to the proofs of the corresponding parts of Theorems~\ref{thm:MainContinuousFNN} and \ref{thm:MainContinuousCNN} respectively, except they use Lemma~\ref{lem:JLfunctionHolder} instead of Lemma~\ref{lem:JLfunctionContinuous}. 

\section{Theorems 1-4 cannot be generalized to differentiable functions}
\label{sec:Example}
Fix positive integers $D \ge 3$ and $d < D$. Define a set $\calS \subset \R^D$ by $\calS = \{\vx \in [-2,2]^D : \|\vx\|_0 \le 2\}$, define the function $f : \calS \to \R$ given by $f(\vx) = \|\vx\|_2^2$, which is smooth. Now, suppose there exists a $\rho$-JL embedding $\mA \in \R^{d \times D}$ of $\calS$ into some hypercube $[-M,M]^d$ (with $d < D$) and a differentiable function $g : [-M,M]^d \to \R$ which satisfies $g(\mA\vx) = f(\vx)$ for all $\vx \in \calS$. 

Then, for any differentiable path $\phi(t) \in \calS$, we must have $$g(\mA\phi(t)) = f(\phi(t)) = \|\phi(t)\|_2^2$$ and thus, $$\nabla g(\mA\phi(t))^T\mA\phi'(t) = 2\phi(t)^T\phi'(t).$$ Let $\ve_1,\ldots,\ve_D \in \R^D$ be the Euclidean basis vectors. For any indices $i,j \in \{1,\ldots,D\}$, we can apply the above result for the differentiable path $\phi(t) = \ve_i+t\ve_j$ and evaluate at $t = 0$ to obtain $$\nabla g(\mA\phi(0))^T\mA\phi'(0) = 2\phi(0)^T\phi'(0)$$
$$\nabla g(\mA \ve_i)^T\mA\ve_j = 2\ve_i^T\ve_j$$ Since this holds for all indices $i,j \in \{1,\ldots,D\}$, we have that $$\nabla g(\mA \ve_i)^T\mA = 2\ve_i^T,$$ for all indices $i \in \{1,\ldots,D\}$, which implies that $2\ve_i$ is in the rowspace of $\mA$ for all $i \in \{1,\ldots,D\}$. However, this is impossible since $\mA$ has $d < D$ rows. Hence, there cannot exist a $\rho$-JL embedding $\mA \in \R^{d \times D}$ of $\calS$ into some hypercube $[-M,M]^d$ (with $d < D$) and a differentiable function $g : [-M,M]^d \to \R$ which satisfies $g(\mA\vx) = f(\vx)$ for all $\vx \in \calS$. 

\section{Proof of Proposition~\ref{prop:JLFiniteToInfinite}}
\label{sec:ProofJLFiniteToInfinite}
Consider a covering of $U_{\calS}$ by $\calN(U_{\calS},\|\cdot\|_2,\delta)$ balls of radius $\delta$. Each ball must intersect $U_{\calS}$ as otherwise we could remove that ball from the covering and obtain a covering of $U_{\calS}$ with only $\calN(U_{\calS},\|\cdot\|_2,\delta)-1$ balls of radius $\delta$, which contradicts the definition of $\calN(U_{\calS},\|\cdot\|_2,\delta)$. Enumerate these balls $i = 1,\ldots,\calN(U_{\calS},\|\cdot\|_2,\delta)$. For each $i$, pick a point $\vu_i \in U_{\calS}$ which is also in the $i$-th ball, and then pick points $\vx_i,\vx'_i \in \calS$ with $\vx_i \neq \vx'_i$ such that $\tfrac{\vx_i-\vx'_i}{\|\vx_i-\vx'_i\|_2} = \vu_i$. Then, set $\calS_1 = \{\vx_i\}_i \cup \{\vx'_i\}_i$ so $|\calS_1|\le 2\mathcal{N}(U_{\calS},\|\cdot\|_2,\delta)$.

Suppose $\mA \in \R^{d \times D}$ is a $\rho$-JL embedding of $\calS_1$. Then, by definition of a $\rho$-JL embedding, $$(1-\rho)\|\vx_i-\vx'_i\|_2 \leq \|\mA\vx_i-\mA\vx'_i\|_2 \leq (1+\rho)\|\vx_i-\vx'_i\|_2, \quad \text{for} \quad i = 1,\ldots,\calN(U_{\calS},\|\cdot\|_2,\delta)$$
Now, for any two points $\vy,\vy' \in \calS$ with $\vy \neq \vy'$, there exists an index $i$ such that $\tfrac{\vy-\vy'}{\|\vy-\vy'\|_2} \in U_{\calS}$ lies in the $i$-th ball of our covering of $U_{\calS}$. Since $\tfrac{\vx_i-\vx'_i}{\|\vx_i-\vx'_i\|_2}$ is also in the $i$-th ball, we have that $$\left \|\frac{\vx_i-\vx_i'}{\|\vx_i-\vx_i'\|_2}- \frac{\vy-\vy'}{\|\vy-\vy'\|_2}\right\|_2 \leq 2\delta.$$
For simplicity of notation, we set $a=\|\vx_i-\vx_i'\|_2, b = \|\vy-\vy'\|_2$. 
Then we immediately have
\begin{align*}
\|\mA(\vy-\vy')\|_2 &= \left\|\frac{b}{a}\mA(\vx_i-\vx_i')+ \mA(\vy-\vy')-\frac{b}{a}\mA(\vx_i-\vx_i')\right\|_2 \\
& \leq \frac{b}{a} \left\|\mA(\vx_i-\vx_i')\right\|_2 + \left\|\mA(\vy-\vy'-\frac{b}{a}(\vx_i-\vx_i'))\right\|_2 \\
&\leq \frac{b}{a}(1+\rho)\|\vx_i-\vx_i'\|_2 + \|\mA\|_2 \left\|\vy-\vy'-\frac{b}{a}(\vx_i-\vx_i')\right\|_2 \\
& \leq (1+\rho)b + 2\|\mA\|_2\delta b = (1+\rho +2\|\mA\|_2\delta)\|\vy-\vy'\|_2,
\end{align*}
where the second inequality used the previous two formulae.
The other side of the bi-Lipschitz formula can be proved similarly. Hence, $\mA$ is also a $(\rho+2\|\mA\|_2\delta)$-JL embedding of $\calS$.

\section{Proof of Proposition~\ref{prop:JLCoveringNumber}}
\label{sec:ProofJLCoveringNumber}
a) By Proposition~\ref{prop:JLFiniteToInfinite}, there exists a finite set $\calS_1$ with at most $|\calS_1| \le 2\calN(U_{\calS},\|\cdot\|_2,\tfrac{\rho}{4\sqrt{3D}})$ points such that any $\tfrac{\rho}{2}$-JL embedding of $\calS_1$ is also a $(\tfrac{\rho}{2}+\|\mA\|_2\tfrac{\rho}{2\sqrt{3D}})$-JL embedding of $\calS$. 

We now show that there exists a matrix $\mA \in \R^{d \times D}$ with $\|\mA\|_2 \le \sqrt{3D}$ which is $\tfrac{\rho}{2}$-JL embedding of $\calS_1$ by generating a random $\mA$ and showing that the probability of $\|\mA\|_2 \le \sqrt{3D}$ and $\mA$ is a $\tfrac{\rho}{2}$-JL embedding of $\calS_1$ both occurring is greater than zero. 

Let $\mA \in \R^{d \times D}$ be a random matrix whose entries are i.i.d. from a subgaussian distribution with mean $0$ and variance $\tfrac{1}{d}$. Since $$\mathbb{E}\|\mA\|_F^2 = \sum_{i = 1}^{d}\sum_{j = 1}^{D}\mathbb{E}\mA_{i,j}^2 = \sum_{i = 1}^{d}\sum_{j = 1}^{D}\dfrac{1}{d} = D,$$ we have that $$\mathbb{P}\left\{\|\mA\|_F^2 \ge 3D\right\} \le \dfrac{\mathbb{E}\|\mA\|_F^2}{3D} = \dfrac{1}{3}.$$ Furthermore, since $$d \gtrsim \rho^{-2}\log \calN(U_{\calS},\|\cdot\|_2,\tfrac{\rho}{4\sqrt{3D}}) \gtrsim \left(\dfrac{\rho}{2}\right)^{-2}\log(3|\calS_1|),$$ by Proposition~\ref{prop:SubgaussianJL}, $\mA$ is a $\tfrac{\rho}{2}$-JL embedding of $\calS_1$ with probability at least $1-\tfrac{1}{3} = \tfrac{2}{3}$. Therefore, $\mA$ is both a $\tfrac{\rho}{2}$-JL embedding of $\calS_1$ and satisfies $\|\mA\|_2 \le \|\mA\|_F \le \sqrt{3D}$ with probability at least $\tfrac{2}{3}-\tfrac{1}{3} = \tfrac{1}{3} > 0$. 

Hence, there exists a matrix $\mA \in \R^{d \times D}$ such that $\mA$ is a $\tfrac{\rho}{2}$-JL embedding of $\calS_1$ and satisfies $\|\mA\|_2 \le \sqrt{3D}$. Finally, by Proposition~\ref{prop:JLFiniteToInfinite}, since $\mA$ is a $\tfrac{\rho}{2}$-JL embedding of $\calS_1$, it is also a $(\tfrac{\rho}{2}+\|\mA\|_2\tfrac{\rho}{2\sqrt{3D}})$-JL embedding of $\calS$. Since $\|\mA\|_2 \le \sqrt{3D}$, we have $\tfrac{\rho}{2}+\|\mA\|_2\tfrac{\rho}{2\sqrt{3D}} \le \rho$, and thus, $\mA$ is a $\rho$-JL embedding of $\calS$, as desired.

b) Again, by Proposition~\ref{prop:JLFiniteToInfinite}, there exists a finite set $\calS_1$ with at most $|\calS_1| \le 2\calN(U_{\calS},\|\cdot\|_2,\tfrac{\rho}{4\sqrt{3D}})$ points such that any $\tfrac{\rho}{2}$-JL embedding of $\calS_1$ is also a $(\tfrac{\rho}{2}+\|\mA\|_2\tfrac{\rho}{2\sqrt{3D}})$-JL embedding of $\calS$. 

Let $\mA \in \R^{d \times D}$ be a random matrix of the form $\mM\mD$ where $\mD \in \R^{D \times D}$ is a diagonal matrix whose entries are independent Rademacher random variables, and $\mM \in \R^{d \times D}$ is a random circulant matrix whose entries are Gaussian random variables with mean $0$ and variance $\tfrac{1}{d}$ and entries in different diagonals are independent. Again, we can show that $$\mathbb{E}\|\mA\|_F^2 = \sum_{i = 1}^{d}\sum_{j = 1}^{D}\mathbb{E}\mA_{i,j}^2 = \sum_{i = 1}^{d}\sum_{j = 1}^{D}\mathbb{E}\mM_{i,j}^2\mD_{j,j}^2 = \sum_{i = 1}^{d}\sum_{j = 1}^{D}\mathbb{E}\mM_{i,j}^2 = \sum_{i = 1}^{d}\sum_{j = 1}^{D}\dfrac{1}{d} = D,$$ and so, $$\mathbb{P}\left\{\|\mA\|_F^2 \ge 3D\right\} \le \dfrac{\mathbb{E}\|\mA\|_F^2}{3D} = \dfrac{1}{3}.$$ Now, set $\alpha = \log(\log(4D+4d))/\log(\log|\calS_1|)$ so that $\log^{\alpha}|\calS_1| = \log(4D+4d)$. Then, since $$d \gtrsim \rho^{-2}\log(4D+4d)\log\calN(U_{\calS},\|\cdot\|_2,\tfrac{\rho}{4\sqrt{3D}}) \gtrsim \left(\dfrac{\rho}{2}\right)^{-2}\log^{1+\alpha}|\calS_1|,$$ by Proposition~\ref{prop:CirculantJL}, $\mA$ is a $\tfrac{\rho}{2}$-JL embedding of $\calS_1$ with probability at least $$\dfrac{2}{3}\left(1-(D+d)e^{-\log^{\alpha}|\calS_1|}\right) = \dfrac{2}{3}\left(1-(D+d)e^{-\log(4D+4d)}\right) = \dfrac{2}{3}\left(1-\dfrac{1}{4}\right) = \dfrac{1}{2}.$$ Therefore, $\mA$ is both a $\tfrac{\rho}{2}$-JL embedding of $\calS_1$ and satisfies $\|\mA\|_2 \le \|\mA\|_F \le \sqrt{3D}$ with probability at least $\tfrac{1}{2}-\tfrac{1}{3} = \tfrac{1}{6} > 0$. 

Hence, there exists a matrix $\mA \in \R^{d \times D}$ such that $\mA$ is a $\tfrac{\rho}{2}$-JL embedding of $\calS_1$ and satisfies $\|\mA\|_2 \le \sqrt{3D}$. Again, by Proposition~\ref{prop:JLFiniteToInfinite}, since $\mA$ is a $\tfrac{\rho}{2}$-JL embedding of $\calS_1$, it is also a $(\tfrac{\rho}{2}+\|\mA\|_2\tfrac{\rho}{2\sqrt{3D}})$-JL embedding of $\calS$. Since $\|\mA\|_2 \le \sqrt{3D}$, we have $\tfrac{\rho}{2}+\|\mA\|_2\tfrac{\rho}{2\sqrt{3D}} \le \rho$, and thus, $\mA$ is a $\rho$-JL embedding of $\calS$, as desired.

\section{Proof of Proposition~\ref{prop:ContinuousFNNlowdim}}
\label{sec:ProofContinuousFNNlowdim}

We first construct a function $\widehat{g}$ that is an $\eps$-approximation of $g$. To do this, we first define a compactly supported ``spike'' function $\phi : \R^d \to [0,1]$ by $$\phi(\vz) = \max\left\{1+\min\left\{\vz_1,\ldots,\vz_d,0\right\}-\max\left\{\vz_1,\ldots,\vz_d,0\right\} ,0\right\}.$$

Then, for any positive integer $N$, define an approximation $\widehat{g} : [-M,M]^d \to \R^p$ to $g$ by $$\widehat{g}(\vy) := \sum_{\vn \in \{-N,\ldots,N\}^d}g(\tfrac{M\vn}{N})\phi(\tfrac{N\vy}{M}-\vn).$$

Similarly to what was done in \citep{Yarotsky18}, it can be shown that the scaled and shifted spike functions $\{\phi(\tfrac{N\vy}{M}-\vn)\}_{\vn \in \{-N,\ldots,N\}^d}$ form a partition of unity, i.e. $$\sum_{\vn \in \{-N,\ldots,N\}^d}\phi(\tfrac{N\vy}{M}-\vn) = 1 \quad \text{for all} \quad \vy \in [-M,M]^d.$$ Trivially, $\phi(\vy) \ge 0$ for all $\vy \in \R^d$. Also, one can check that $\supp(\phi) \subseteq [-1,1]^d$, and thus, $\phi(\tfrac{N\vy}{M}-\vn) = 0$ for all $\vn$ such that $\|\tfrac{N\vy}{M}-\vn\|_{\infty} > 1$. Furthermore, for any $\vn$ such that $\|\tfrac{N\vy}{M}-\vn\|_{\infty} \le 1$, we have $$\left\|g(\vy)-g(\tfrac{M\vn}{N})\right\|_2 \le \Delta\left(\|\vy-\tfrac{M\vn}{N}\|_2\right) \le \Delta\left(\sqrt{d}\|\vy-\tfrac{M\vn}{N}\|_{\infty}\right) = \Delta\left(\tfrac{M\sqrt{d}}{N}\|\tfrac{N\vy}{M}-\vn\|_{\infty}\right) \le \Delta \left(\tfrac{M\sqrt{d}}{N}\right).$$ Hence, we can bound the approximation error for any $\vy \in [-M,M]^d$ as follows:

\begin{align*}
\left\|\widehat{g}(\vy)-g(\vy)\right\|_2 &= \left\|\sum_{\vn \in \{-N,\ldots,N\}^d}g(\tfrac{M\vn}{N})\phi(\tfrac{N\vy}{M}-\vn) - g(\vy)\right\|_2
\\
&= \left\|\sum_{\vn \in \{-N,\ldots,N\}^d}\left(g(\tfrac{M\vn}{N})-g(\vy)\right)\phi(\tfrac{N\vy}{M}-\vn)\right\|_2
\\
&\le \sum_{\vn \in \{-N,\ldots,N\}^d}\left\|g(\tfrac{M\vn}{N})-g(\vy)\right\|_2\phi(\tfrac{N\vy}{M}-\vn)
\\
&= \sum_{\left\|\tfrac{N\vy}{M}-\vn\right\|_{\infty} \le 1}\left\|g(\tfrac{M\vn}{N})-g(\vy)\right\|_2\phi(\tfrac{N\vy}{M}-\vn)
\\
&\le \sum_{\left\|\tfrac{N\vy}{M}-\vn\right\|_{\infty} \le 1}\Delta \left(\tfrac{M\sqrt{d}}{N}\right)\phi(\tfrac{N\vy}{M}-\vn)
\\
&\le \sum_{\vn \in \{-N,\ldots,N\}^d}\Delta \left(\tfrac{M\sqrt{d}}{N}\right)\phi(\tfrac{N\vy}{M}-\vn)
\\
&= \Delta \left(\tfrac{M\sqrt{d}}{N}\right).
\end{align*}
So $\|\widehat{g}(\vy)-g(\vy)\|_{\infty} \le \|\widehat{g}(\vy)-g(\vy)\|_2 \le \Delta \left(\tfrac{M\sqrt{d}}{N}\right)$ for all $\vy \in [-M,M]^d$, i.e., $\widehat{g}$ is a $\Delta \left(\tfrac{M\sqrt{d}}{N}\right)$-approximation of $g$.

We now focus on constructing a ReLU NN architecture which can implement the $\eps$-approximation $\widehat{g}$ for any function $g$ that admits $\Delta(r)$ as a modulus of continuity. We do this by first constructing a ReLU NN that is independent of $g$ which implements the map $\Phi : \R^d \to \R^{(2N+1)^d}$ defined by $(\Phi(\vy))_{\vn} = \phi(\tfrac{N\vy}{M}-\vn)$. Then, we add a final layer which outputs the appropriate linear combination of the $\phi(\tfrac{N\vy}{M}-\vn)$'s.

\begin{lem}
For any integers $N,d \ge 1$, the maps $m_d : \R^d \to \R^{(2N+1)^d}$ and $M_d : \R^d \to \R^{(2N+1)^d}$ defined by $$(m_d(\vy))_{\vn} := \min\left\{\tfrac{N\vy_1}{M}-\vn_1,\ldots,\tfrac{N\vy_d}{M}-\vn_d,0\right\} \quad \text{for} \quad \vn \in \{-N,\ldots,N\}^d$$ and $$(M_d(\vy))_{\vn} := \max\left\{\tfrac{N\vy_1}{M}-\vn_1,\ldots,\tfrac{N\vy_d}{M}-\vn_d,0\right\} \quad \text{for} \quad \vn \in \{-N,\ldots,N\}^d,$$ can both be implemented by a ReLU NN with $O((2N+1)^d)$ weights, $O((2N+1)^d)$ nodes, and $\ceil{\log_2(d+1)}$ layers.
\end{lem}

\begin{proof}
First, we note that we can write 
\begin{align*}
(m_d(\vy))_{\vn} =& \min\left\{\min\{\tfrac{N\vy_1}{M}-\vn_1,\ldots,\tfrac{N\vy_{\ceil{d/2}}}{M}-\vn_{\ceil{d/2}}\}, \right.
\\
& \ \ \ \ \ \ \ \ \ \ \left.\min\{\tfrac{N\vy_{\ceil{d/2}+1}}{M}-\vn_{\ceil{d/2}+1},\ldots,\tfrac{N\vy_d}{M}-\vn_d,0\}\right\}
\end{align*} and \begin{align*}
(M_d(\vy))_{\vn} =& \max\left\{\max\{\tfrac{N\vy_1}{M}-\vn_1,\ldots,\tfrac{N\vy_{\ceil{d/2}}}{M}-\vn_{\ceil{d/2}}\}, \right.
\\
& \ \ \ \ \ \ \ \ \ \ \left.\max\{\tfrac{N\vy_{\ceil{d/2}+1}}{M}-\vn_{\ceil{d/2}+1},\ldots,\tfrac{N\vy_d}{M}-\vn_d,0\}\right\}
\end{align*} 

In \citep{arora2016understanding}, it is shown that for any positive integer $k$, the maps $(\vz_1,\ldots,\vz_k) \mapsto \min\{\vz_1,\ldots,\vz_k\}$ and $(\vz_1,\ldots,\vz_k)\mapsto \max\{\vz_1,\ldots,\vz_k\}$ can be implemented by a ReLU NN with at most $c_1k$ edges, $c_2k$ nodes, and $\ceil{\log_2 k}$ layers, where $c_1, c_2 > 0$ are universal constants. So to construct the map $m_d$, we first implement the $(2N+1)^{\ceil{d/2}}$ maps \begin{equation}\label{eq:firsthalf}(\vy_1,\ldots,\vy_{\ceil{d/2}}) \mapsto \min\{\tfrac{N\vy_1}{M}-\vn_1,\ldots,\tfrac{N\vy_{\ceil{d/2}}}{M}-\vn_{\ceil{d/2}}\}\end{equation} for $(\vn_1,\ldots,\vn_{\ceil{d/2}}) \in \{-N,\ldots,N\}^{\ceil{d/2}}$. Implementing each of these maps requires $c_1\ceil{\tfrac{d}{2}}$ edges, $c_2\ceil{\tfrac{d}{2}}$ nodes, and $\ceil{\log_2\ceil{\tfrac{d}{2}}}$ layers. Next, we implement the $(2N+1)^{\floor{d/2}}$ maps \begin{equation}\label{eq:secondhalf}(\vy_{\ceil{d/2}+1},\ldots,\vy_d) \mapsto \min\{\tfrac{N\vy_{\ceil{d/2}+1}}{M}-\vn_{\ceil{d/2}+1},\ldots,\tfrac{N\vy_d}{M}-\vn_d,0\}\end{equation} for $(\vn_{\ceil{d/2}+1},\ldots,\vn_d) \in \{-N,\ldots,N\}^{\floor{d/2}}$. Implementing each of these maps requires $c_1(\floor{\tfrac{d}{2}}+1)$ edges, $c_2(\floor{\tfrac{d}{2}}+1)$ nodes, and $\ceil{\log_2(\floor{\tfrac{d}{2}}+1)}$ layers. After placing these $(2N+1)^{\ceil{d/2}}+(2N+1)^{\floor{d/2}}$ maps in parallel, we construct one final layer as follows. For each $\vn = (\vn_1,\ldots,\vn_d) \in \{-N,\ldots,N\}^d$, we combine the output of the $(\vn_1,\ldots,\vn_{\ceil{d/2}})$-th map of the form in Equation~\ref{eq:firsthalf} and the output of the $(\vn_{\ceil{d/2}+1},\ldots,\vn_d)$-th map of the form in Equation~\ref{eq:secondhalf} by using them as inputs to a ReLU NN that implements the map $(a,b) \mapsto \min\{a,b\}$. Each of these requires at most $2c_1$ edges and $2c_2$ nodes.

The total number of edges used to implement $m_d$ is \begin{align*}
&c_1\ceil{\tfrac{d}{2}}(2N+1)^{\ceil{d/2}} + c_1(\floor{\tfrac{d}{2}}+1)(2N+1)^{\floor{d/2}} + 2c_1(2N+1)^d
\\
\le & c_1(\ceil{\tfrac{d}{2}}+\floor{\tfrac{d}{2}}+1)(2N+1)^{\ceil{d/2}}+2c_1(2N+1)^d
\\
= & c_1(d+1)(2N+1)^{\ceil{d/2}}+2c_1(2N+1)^d
\\
= & c_1\left((d+1)(2N+1)^{-\floor{d/2}}+2\right)(2N+1)^d
\\
\le & c_1\left((d+1)\cdot 3^{-\floor{d/2}}+2\right)(2N+1)^d
\\
\le & 4c_1(2N+1)^d,
\end{align*} where we have used the fact that $N \ge 1$ by definition, and the easily verifiable inequality $(d+1)\cdot 3^{-\floor{d/2}} \le 2$ for all positive integers $d$.

A nearly identical calculation shows that the total number of nodes used to implement $m_d$ is at most $4c_2(2N+1)^d$. Finally, since the $(2N+1)^{\ceil{d/2}}$ maps of the form in Equation~\ref{eq:firsthalf} and the $(2N+1)^{\floor{d/2}}$ maps of the form in Equation~\ref{eq:secondhalf} are in parallel, the total number of layers used to implement $m_d$ is $$\max\left\{\ceil{\log_2\ceil{\tfrac{d}{2}}},\ceil{\log_2(\floor{\tfrac{d}{2}}+1)}\right\}+1 = \ceil{\log_2(d+1)}.$$

Hence, the map $m_d$ can be implemented by a ReLU NN with at most $C_1(2N+1)^d$ edges, $C_2(2N+1)^d$ nodes, and $\ceil{\log_2(d+1)}$ layers, as desired. The proof for $M_d$ is identical, except with $\min$ replaced by $\max$. 
\end{proof}

Next, we note that $$(\Phi(\vy))_{\vn} = \phi(\tfrac{N\vy}{M}-\vn) = \max\left\{1+(m_d(\vy))_{\vn}-(M_d(\vy))_{\vn},0\right\} \quad \text{for all} \quad \vn \in \{-N,\ldots,N\}^d.$$ So to construct a ReLU NN which implements $\Phi$, we first place a ReLU NN that implements $m_d$ in parallel with a ReLU NN that implements $M_d$. Then, we add an extra layer which has $(2N+1)^d$ nodes, where the $\vn$-th node of this layer has two edges, one from the $\vn$-th node of $m_d$ and one from the $\vn$-th node of $M_d$. Since $m_d$ and $M_d$ are in parallel and each can each be implemented with ReLU NNs with $O((2N+1)^d)$ edges, $O((2N+1)^d)$ nodes, and $\ceil{\log_2(d+1)}$ layers, and the last layer has $2(2N+1)^d$ edges and $(2N+1)^d$ nodes, the ReLU NN which implements $\Phi$ has $O((2N+1)^d)$ edges, $O((2N+1)^d)$ nodes, and $\ceil{\log_2(d+1)}+1$ layers.

Finally, we can construct a ReLU NN which implements $$\widehat{g}(\vx) := \sum_{\vn \in \{-N,\ldots,N\}^d}g(\tfrac{M\vn}{N})\phi(\tfrac{N\vx}{M}-\vn)$$ by using the ReLU NN which implements $\Phi$, followed by a linear layer which computes the weighted sum for $\widehat{g}$. This last layer has $p$ nodes, and $p(2N+1)^d$ edges. So the ReLU NN that implements $\widehat{g}$ has $(p+C_1)(2N+1)^d$ edges, $C_2(2N+1)^d+p$ nodes, and $\ceil{\log_2(d+1)}+2$ layers, as desired.

\section{Proof of Theorem~\ref{thm:ContinuousFNN}}
\label{sec:ProofContinuousFNN}

By combining Proposition~\ref{prop:JLCoveringNumber}a and Proposition~\ref{prop:JLGaussianWidth}, we have that there exists a $\rho$-JL embedding $\mA \in \R^{d \times D}$ of $\calS$ with $$d \gtrsim \min\left\{ \rho^{-2}\log  \mathcal{N}(U_{\mathcal{S}}, \|\cdot\|_2, \tfrac{\rho}{4\sqrt{3D}}),\ \ \rho^{-2}\left(\omega(U_{\calS})\right)^2\right\}.$$ Let $M = \sup_{\vx \in \calS}\|\mA\vx\|_{\infty}$ so that $\mA(\calS) \subset [-M,M]^d$, and so, $\mA$ is a $\rho$-JL embedding of $\calS$ into $[-M,M]^d$. By Proposition~\ref{prop:ContinuousFNNlowdim}, there exists a ReLU NN architecture with at most $$\calE = (p+C_1)(2N+1)^d \ \text{edges},$$ 
$$\calN = C_2(2N+1)^d + p \ \text{nodes},$$
$$\text{and} \ \calL = \ceil{\log_2(d+1)} + 2 \ \text{layers},$$ which can $\sqrt{p}\Delta(\tfrac{M\sqrt{d}}{(1-\rho)N})$-approximate any function $g : [-M,M]^d \to \R^p$ which admits a modulus of continuity $\sqrt{p}\Delta(\tfrac{r}{1-\rho})$. Finally, by applying Theorem~\ref{thm:MainContinuousFNN}b, we have that there exists a a ReLU NN architecture with at most $$\calE+Dd = (p+C_1)(2N+1)^d + Dd \ \text{edges},$$ 
$$\calN+D = C_2(2N+1)^d + p + D \ \text{nodes},$$
$$\text{and} \ \calL + 1 = \ceil{\log_2(d+1)} + 3 \ \text{layers},$$ which can $\sqrt{p}\Delta(\tfrac{M\sqrt{d}}{(1-\rho)N})$-approximate any function $f : \calS \to \R^p$ that admits a modulus of continuity of $\Delta(r)$, as desired.

\section{Proof of Theorem~\ref{thm:HolderFNN}}
\label{sec:ProofHolderFNN}

By combining Proposition~\ref{prop:JLCoveringNumber}a and Proposition~\ref{prop:JLGaussianWidth}, we have that there exists a $\rho$-JL embedding $\mA \in \R^{d \times D}$ of $\calS$ with $$d \gtrsim \min\left\{ \rho^{-2}\log  \mathcal{N}(U_{\mathcal{S}}, \|\cdot\|_2, \tfrac{\rho}{4\sqrt{3D}}),\ \ \rho^{-2}\left(\omega(U_{\calS})\right)^2\right\}.$$ Let $M = \sup_{\vx \in \calS}\|\mA\vx\|_{\infty}$ so that $\mA(\calS) \subset [-M,M]^d$, and so, $\mA$ is a $\rho$-JL embedding of $\calS$ into $[-M,M]^d$. By Corollary~\ref{cor:HolderFNNlowdim}, there exists a ReLU NN architecture with at most $$\calE = (p+C_1)\left(2\ceil{\dfrac{M\sqrt{d}}{(1-\rho)(\eps/L)^{1/\alpha}}}+1\right)^d \ \text{edges},$$ 
$$\calN = C_2\left(2\ceil{\dfrac{M\sqrt{d}}{(1-\rho)(\eps/L)^{1/\alpha}}}+1\right)^d + p \ \text{nodes},$$
$$\text{and} \ \calL = \ceil{\log_2(d+1)} + 2 \ \text{layers},$$ which can $\eps$-approximate any $(\tfrac{L}{(1-\rho)^{\alpha}},\alpha)$-H\"{o}lder function $g : [-M,M]^d \to \R^p$. Finally, by applying Theorem~\ref{thm:MainHolderFNN}b, we have that there exists a ReLU NN architecture with at most $$\calE+Dd = (p+C_1)\left(2\ceil{\dfrac{M\sqrt{d}}{(1-\rho)(\eps/L)^{1/\alpha}}}+1\right)^d + Dd \ \text{edges},$$ 
$$\calN+D = C_2\left(2\ceil{\dfrac{M\sqrt{d}}{(1-\rho)(\eps/L)^{1/\alpha}}}+1\right)^d + p + D \ \text{nodes},$$
$$\text{and} \ \calL + 1 = \ceil{\log_2(d+1)} + 3 \ \text{layers},$$ which can $\eps$-approximate any $(L,\alpha)$-H\"{o}lder function $f : \calS \to \R^p$, as desired.

\section{Proof of Theorem~\ref{thm:HolderCNN}}
\label{sec:ProofHolderCNN}

Let $f : \calS \to \R^p$ be the $\alpha$-H\"{o}lder target function to approximate. By Proposition~\ref{prop:JLCoveringNumber}b, we have that there exists a matrix $\mA \in \R^{d \times D}$ in the form $\mM\mD$ where $\mM$ is a partial circulant matrix and $\mD$ is a diagonal matrix with $\pm 1$ entries such that $\mA$ is a $\rho$-JL embedding of $\calS$ with $$d \gtrsim \rho^{-2}\log(4D+4d)\log \mathcal{N}(U_{\mathcal{S}}, \|\cdot\|_2, \tfrac{\rho}{4\sqrt{3D}}).$$ Let $M = \sup_{\vx \in \calS}\|\mA\vx\|_{\infty}$ so that $\mA(\calS) \subset [-M,M]^d$, and so, $\mA$ is a $\rho$-JL embedding of $\calS$ into $[-M,M]^d$. 

By Lemma~\ref{lem:JLfunctionHolder}, there exists an $\alpha$-H\"{o}lder function $g : [-M,M]^d \to \R^p$ such that $g(\mA\vx) = f(\vx)$ for all $\vx \in \calS$. Let $g_i : [-M,M]^d \to \R$ be the $i$-th coordinate of $g$. Let $\widetilde{g}_i : [-1,1]^d \to \R$ be defined by $\widetilde{g}_i(\vy) = g_i(M\vy)$ for all $\vy \in [-1,1]^d$. Note that each $g_i$ is $\alpha$-H\"{o}lder, and so, each $\widetilde{g}_i$ is also $\alpha$-H\"{o}lder.

Then, by Proposition~\ref{prop:HolderCNNlowdim}, for each $\widetilde{g}_i$, there exists a CNN $\widetilde{g}_i^{(CNN)}$ with $O(N)$ residual blocks, each of which has depth $O(\log N)$ and $O(1)$ channels, and whose filter size is at most $K$ such that $\|\widetilde{g}_i-\widetilde{g}_i^{(CNN)}\|_{\infty} \le \widetilde{O}(N^{-\alpha/d})$.

Now, we construct a CNN to approximate $f$ as follows. First, we implement the map $\vx \mapsto \tfrac{1}{M}\mA\vx$ using the same $4$ layer ResNet CNN described in the proof of Theorem~\ref{thm:MainHolderCNN}c. Then, we pass the output of that Resnet CNN into $p$ parallel CNNs which implement $\widetilde{g}_i^{(CNN)}$ for $i = 1,\ldots,p$. The output of the $i$-th of these parallel CNNs is $\widetilde{g}_i^{(CNN)}(\tfrac{1}{M}\mA\vx)$, which is an $\widetilde{O}(N^{-\alpha/d})$-approximation of $\widetilde{g}_i(\tfrac{1}{M}\mA\vx) = g_i(\mA\vx) = f_i(\vx)$. Hence, the constructed CNN is a $\widetilde{O}(N^{-\alpha/d})$-approximation of $f$. 

The CNN which implements the map $\vx \mapsto \tfrac{1}{M}\mA\vx$ needs $O(1)$ residual blocks, each of which has depth $O(1)$ and $O(1)$ channels. Each of the $p$ parallel CNNs which implement the $\widetilde{g}_i^{(CNN)}$'s have $O(N)$ residual blocks, each of which has depth $O(\log N)$ and $O(1)$ channels. So the overall network to approximate $f$ has $O(pN)$ residual blocks, each of which has depth $O(\log N)$ and $O(1)$ channels.

\section{Proof of Proposition~\ref{prop:SparseRecovery}}
\label{sec:ProofSparseRecovery}

\begin{proof} By definition, the unit secants of $\mathcal{Y}=\mPhi(\Sigma_s^N)$ are defined to be
\[
U_{\mathcal{Y}} =\left\{ \frac{\vy_1-\vy_2}{\|\vy_1-\vy_2\|_2}, \quad \vy_1,\vy_2 \in \mPhi(\Sigma_s^N)\right\} 
\]
which contains all unit vectors that are linear combinations of $2s$ columns of $\mPhi$.
Letting $T$ with $|T| = 2s$ be a fixed support set, the covering number of $\text{span}(\mPhi_T)\cap \mathbb{S}^{N-1}$ is $(\frac{3}{\delta})^{2s}$, so the covering number of $U_{\mathcal{Y}}$ is at most $\binom{N}{2s}(\frac{3}{\delta})^{2s} \le (\tfrac{eN}{2s})^{2s}(\frac{3}{\delta})^{2s}$. Therefore, $$\log \mathcal{N}(U_{\mathcal{Y}},\|\cdot\|_2,\delta)\lesssim  s \log \frac{N}{s\delta}.$$

Now, let $\calS$ be any bounded subset of $\calY$. Note that the inverse map $F^{-1}$ is $L$-Lipschitz (i.e. $(L,\alpha)$-H\"{o}lder for $\alpha = 1$). Since $\calS \subset \calY$, we have that $U_{\calS} \subset U_{\calY}$, and thus, $\calN(U_{\calS},\|\cdot\|_2,\delta) \le \calN(U_{\calY},\|\cdot\|_2,\delta)$ for any $\delta > 0$. Therefore, the choice of $d$ in this proposition yields $$d \gtrsim \rho^{-2}s\log\dfrac{N\sqrt{m}}{s\rho} \gtrsim \rho^{-2}\log\calN(U_{\calY},\| \cdot \|_2,\tfrac{\rho}{4\sqrt{3m}}) \ge \rho^{-2}\log\calN(U_{\calS},\| \cdot \|_2,\tfrac{\rho}{4\sqrt{3m}}).$$

Therefore, the conditions of Theorem~\ref{thm:HolderFNN} are satisfied for $\alpha = 1$, and the result follows.

\end{proof}

\section{Proof of Proposition~\ref{prop:BlindDeconvolution}}
\label{sec:ProofBlindDeconvolution}

We first bound the logarithm of the covering number of the set of unit secants of $\calY$. Then, we apply Theorem~\ref{thm:HolderFNN} for the case of $\alpha = 1$, i.e. Lipschitz continuous functions.

By the $\sin\Theta$ theorem \citep{wedin1972perturbation}, we have 
\begin{equation}\label{eq:sin1}
\left\|\frac{\vx_1}{\|\vx_1\|_2} - \frac{\vx_2}{\|\vx_2\|_2}\right\|_2 \leq \frac{\|\vx_1\vk_1^T-\vx_2\vk_2^T\|}{\|\vx_1\|_2\|\vk_1\|_2}
\end{equation}
and 
\begin{equation}\label{eq:sin2}
\left\|\frac{\vk_1}{\|\vk_1\|_2} - \frac{\vk_2}{\|\vk_2\|_2}\right\|_2 \leq \frac{\|\vx_1\vk_1^T-\vx_2\vk_2^T\|}{\|\vx_2\|_2\|\vk_2\|_2}.
\end{equation}
Let us find a set whose covering number is easy to compute while containing the unit secant $U_\mathcal{Y}$ as a subset
\begin{align*}
& \left\{ \frac{\vy_1-\vy_2}{\|\vy_1-\vy_2\|_2}, \ \vy_1, \vy_2 \in \mathcal{Y}  \right\} =\left\{ \frac{\vx_1\otimes \vk_1-\vx_2\otimes \vk_2}{\|\vx_1\otimes \vk_1-\vx_2\otimes \vk_2\|_2}, \ \vx_i = \mPhi \vu_i, \vk_i =\mPsi \vv_i, \vx_i\otimes \vk_i\in \mathcal Y, i=1,2 \right\}
\\ & = \left\{ \frac{\vx_1\otimes \vk_1-(\frac{\|\vx_1\|_2}{\|\vx_2\|_2}\vx_2)\otimes \vk_1}{\|\vx_1\otimes \vk_1-\vx_2\otimes \vk_2\|_2} + \frac{(\frac{\|\vx_1\|_2}{\|\vx_2\|_2}\vx_2)\otimes \vk_1-\vx_2\otimes \vk_2}{\|\vx_1\otimes \vk_1-\vx_2\otimes \vk_2\|_2}, \ \vx_i = \mPhi \vu_i, \vk_i =\mPsi \vv_i, \vx_i\otimes \vk_i\in \mathcal Y, i=1,2 \right\}\\
& \subseteq  \left\{ \frac{\vx_1\otimes \vk_1-(\frac{\|\vx_1\|_2}{\|\vx_2\|_2}\vx_2)\otimes \vk_1}{\|\vx_1\otimes \vk_1-\vx_2\otimes \vk_2\|_2},  \ \vx_i = \mPhi \vu_i, \vk_i =\mPsi \vv_i, i=1,2 \right\} \\ &+ \left\{\frac{(\frac{\|\vx_1\|_2}{\|\vx_2\|_2}\vx_2)\otimes \vk_1-\vx_2\otimes \vk_2}{\|\vx_1\otimes \vk_1-\vx_2\otimes \vk_2\|_2},\ \vx_i = \mPhi \vu_i, \vk_i =\mPsi \vv_i,\vx_i\otimes \vk_i\in \mathcal Y, i=1,2 \right\}.
\end{align*}
For the first set in the sum, by using \eqref{eq:lip} and \eqref{eq:sin1}, we have
\begin{align*}
& \left\{ \frac{\vx_1\otimes \vk_1-(\frac{\|\vx_1\|_2}{\|\vx_2\|_2}\vx_2)\otimes \vk_1}{\|\vx_1\otimes \vk_1-\vx_2\otimes \vk_2\|_2}, \  \vx_i = \mPhi \vu_i, \vk_i =\mPsi \vv_i, i=1,2 \right\} \\ &\subseteq \left\{ t\cdot \frac{\vx_1\otimes \vk_1-(\frac{\|\vx_1\|_2}{\|\vx_2\|_2}\vx_2)\otimes \vk_1}{\|\vx_1 \vk_1^T-\vx_2 \vk_2^T\|_2}, \ t\in [0,L], \ \vx_i = \mPhi \vu_i, \vk_i =\mPsi \vv_i,\vx_i\otimes \vk_i\in \mathcal Y, i=1,2 \right\} \\
&\subseteq \left\{ t\cdot \frac{\vx_1\otimes \vk_1-(\frac{\|\vx_1\|}{\|\vx_2\|}\vx_2)\otimes \vk_1}{\|\vx_1-\frac{\|\vx_1\|_2}{\|\vx_2\|_2}\vx_2\|\|\vk_1\|}, \ t\in [0,L], \ \vx_i = \mPhi \vu_i, \vk_i =\mPsi \vv_i, \vx_i\otimes \vk_i\in \mathcal Y,i=1,2 \right\}\\
& \subseteq \left\{ \left( \sqrt t\cdot \frac{\vx_1-\frac{\|\vx_1\|_2}{\|\vx_2\|_2}\vx_2}{\|\vx_1-\frac{\|\vx_1\|_2}{\|\vx_2\|_2}\vx_2\|_2}\right)\otimes \left(\sqrt t \cdot \frac{\vk_1}{\|\vk_1\|_2}\right), t\in [0,L],\  \vx_i = \mPhi \vu_i, \vk_i =\mPsi \vv_i, \vx_i\otimes \vk_i\in \mathcal Y,i=1,2 \right\}
\end{align*}
 
The covering number with $\epsilon$ balls of the set $\left\{\sqrt t\cdot \frac{\vx_1\otimes \vk_1-(\frac{\|\vx_1\|_2}{\|\vx_2\|_2}\vx_2)}{\|\vx_1\otimes \vk_1-(\frac{\|\vx_1\|_2}{\|\vx_2\|_2}\vx_2)\|_2}, t\in [0,L]\right\}$ is $\left(\frac{3\sqrt L}{\epsilon}\right)^n$, and that for the set $\{\sqrt t \cdot \frac{\vk_1}{\|\vk_1\|_2}, t\in[0,L]\}$ is $\left(\frac{3\sqrt L}{\epsilon}\right)^m$. So the covering number with $\epsilon$ balls of $S$ is  $$\left(\frac{6L}{\epsilon}\right)^n+\left(\frac{6L}{\epsilon}\right)^m.$$ The same argument holds for the second set in the sum. Therefore, $$\log \mathcal{N}(U_{\mathcal{Y}},\|\cdot\|_2,\delta)\lesssim  \max\{m,n\} \log \frac{L}{\delta}.$$

Now, let $\calS$ be any bounded subset of $\calY$. Note that the inverse map $F^{-1}$ is $L$-Lipschitz (i.e. $(L,\alpha)$-H\"{o}lder for $\alpha = 1$). Since $\calS \subset \calY$, we have that $U_{\calS} \subset U_{\calY}$, and thus, $\calN(U_{\calS},\|\cdot\|_2,\delta) \le \calN(U_{\calY},\|\cdot\|_2,\delta)$ for any $\delta > 0$. Therefore, the choice of $d$ in this proposition yields $$d \gtrsim \rho^{-2}\max\{m,n\}\log\dfrac{L\sqrt{N}}{\rho} \gtrsim \rho^{-2}\log\calN(U_{\calY},\| \cdot \|_2,\tfrac{\rho}{4\sqrt{3N}}) \ge \rho^{-2}\log\calN(U_{\calS},\| \cdot \|_2,\tfrac{\rho}{4\sqrt{3N}}).$$

Therefore, the conditions of Theorem~\ref{thm:HolderFNN} are satisfied for $\alpha = 1$, and the result follows.

\section{Proof of Proposition~\ref{prop:MatrixCompletion}}
\label{sec:ProofMatrixCompletion}

By definition,
\[
U_\mathcal{Y} = \left\{\frac{\vy_1-\vy_2}{\|\vy_1-\vy_2\|_2}, \vy_1,\vy_2\in \mathcal{Y} \right\} =  \left\{\frac{P_{\Omega}(\mX_1-\mX_2)}{\|P_{\Omega}(\mX_1-\mX_2)\|_F},\mX_1,\mX_2\in \mathcal{Y} \right\}
\]
Since $\vy_1-\vy_2 = P_{\Omega}(\mX_1-\mX_2)$ and $\mX$
\[
\left\{\frac{P_{\Omega}(\mX_1-\mX_2)}{\|P_{\Omega}(\mX_1-\mX_2)\|_F}, \mX_1,\mX_2\in \mathcal{Y} \right\} \subseteq  \left\{t\cdot P_{\Omega}\left( \frac{\mX_1-\mX_2}{\|\mX_1-\mX_2\|_F}\right), t \in [0,L], \mX_1,\mX_2\in \mathcal{Y} \right\}
\]
Notice that $\frac{\mX_1-\mX_2}{\|\mX_1-\mX_2\|_F}$ are matrices of unit Frobenius norm with rank at most $2r$. By Lemma 3.1 in \citep{candes2011tight}, they form a set whose covering number is at most $\left(\frac{9}{\delta}\right)^{r(m+n+1)}$. Hence, by dilating the set by a factor of $L$, the covering number is at most $\left(\frac{9L}{\delta}\right)^{r(m+n+1)}$. Therefore, $$\log \mathcal{N}(U_{\mathcal{Y}},\|\cdot\|_2,\delta)\lesssim  r(m+n) \log \frac{L}{\delta}.$$

Now, let $\calS$ be any bounded subset of $\calY$. Note that the inverse map $F^{-1}$ is $L$-Lipschitz (i.e. $(L,\alpha)$-H\"{o}lder for $\alpha = 1$). Since $\calS \subset \calY$, we have that $U_{\calS} \subset U_{\calY}$, and thus, $\calN(U_{\calS},\|\cdot\|_2,\delta) \le \calN(U_{\calY},\|\cdot\|_2,\delta)$ for any $\delta > 0$. Therefore, the choice of $d$ in this proposition yields $$d \gtrsim \rho^{-2}r(m+n)\log\dfrac{L\sqrt{mn}}{\rho} \gtrsim \rho^{-2}\log\calN(U_{\calY},\| \cdot \|_2,\tfrac{\rho}{4\sqrt{3mn}}) \ge \rho^{-2}\log\calN(U_{\calS},\| \cdot \|_2,\tfrac{\rho}{4\sqrt{3mn}}).$$

Therefore, the conditions of Theorem~\ref{thm:HolderFNN} are satisfied for $\alpha = 1$, and the result follows.
\end{document}